\documentclass{article}

\usepackage{microtype}
\usepackage{graphicx}
\usepackage{subcaption}
\usepackage{booktabs} 

\usepackage{hyperref}

\usepackage[accepted]{icml2026}

\usepackage{amsmath}
\usepackage{amssymb}
\usepackage{mathtools}
\usepackage{amsthm}

\usepackage{enumitem}

\usepackage[capitalize,noabbrev]{cleveref}
\usepackage{pifont}
\theoremstyle{plain}
\newtheorem{theorem}{Theorem}[section]
\newtheorem{proposition}[theorem]{Proposition}
\newtheorem{lemma}[theorem]{Lemma}

\theoremstyle{definition}
\newtheorem{definition}[theorem]{Definition}
\newtheorem{assumption}[theorem]{Assumption}
\theoremstyle{remark}

\newcommand{\Eb}{\mathbb{E}}

\newcommand{\Mc}{\mathcal{M}}

\newcommand{\Ec}{\mathcal{E}}

\newcommand{\Vol}{\mathrm{Vol}}

\ifodd 1
\newcommand{\jing}[1]{{\color{red}#1}}
\else
\newcommand{\jing}[1]{#}
\fi

\ifodd 1
\newcommand{\jingc}[1]{{\color{red}(Jing: #1)}}
\else
\newcommand{\jingc}[1]{}
\fi

\ifodd 1
\newcommand{\jingf}[1]{{\color{green}(Jing: #1)}}
\else
\newcommand{\jingf}[1]{#}
\fi

\usepackage[textsize=tiny]{todonotes}

\icmltitlerunning{Provably Efficient Actor–Critic for
Low-Rank MDPs}

\begin{document}

\twocolumn[

\icmltitle{Breaking the Computational Barrier: Provably Efficient Actor–Critic for Low-Rank MDPs}



  \icmlsetsymbol{equal}{*}

  \begin{icmlauthorlist}
    \icmlauthor{Ruiquan Huang}{uky}
    \icmlauthor{Donghao Li}{uva}
    \icmlauthor{Yingbin Liang}{osu}
    \icmlauthor{Jing Yang}{uva}
  \end{icmlauthorlist}

  \icmlaffiliation{uky}{University of Kentucky.}
  \icmlaffiliation{uva}{University of Virginia.}
  \icmlaffiliation{osu}{Ohio State University.}

  \icmlcorrespondingauthor{Ruiquan Huang}{rhu285@uky.edu}

  \icmlkeywords{Machine Learning, ICML}

  \vskip 0.3in
]



\printAffiliationsAndNotice{}  

\begin{abstract}
Reinforcement learning (RL) is a fundamental framework for sequential decision-making, in which an agent learns an optimal policy through interactions with an unknown environment. In settings with function approximation, many existing RL algorithms achieve favorable sample complexity, but often rely on computationally intractable oracles. In this paper, we use supervised learning as a computational proxy to establish a clear hierarchy of commonly adopted RL oracles under low-rank Markov Decision Processes (MDPs). This hierarchy shows that policy evaluation is the most computationally efficient oracle, provided that supervised learning can be efficiently solved. Motivated by this observation, we propose a novel optimistic actor–critic algorithm that relies solely on a truncated policy evaluation oracle. We prove that our algorithm outperforms the existing sample complexity guarantees for low-rank MDPs, while avoiding computationally expensive planning or optimization oracles commonly assumed in prior works. We further demonstrate that approximately low-rank MDPs capture a broad class of real world environments. Finally, we evaluate a practical variant of our algorithm with experiments on several standard Gym benchmarks.
\end{abstract}

\section{Introduction}

In reinforcement learning (RL), an agent interacts sequentially with an unknown environment to optimize its decision-making strategy. 
A central challenge in RL is sample efficiency, which measures the number of interactions required to identify an optimal or near-optimal policy. Understanding sample complexity has been a core objective of RL theory, with extensive progress made under a variety of settings, including tabular MDPs with finite state–action spaces~\citep{agrawal2017optimistic,domingues2020episodic,menard2021ucb}, Linear MDPs~\cite{jin2020provably,wang2020reward}, and more general function approximation frameworks~\citep{jiang2017contextual,jin2021bellman,liu2022optimistic}, where the state space may be infinite and the transition dynamics belong to rich function classes.

As deep learning continues to advance and apply to more complex real-world environments, RL with function approximation has become essential. However, in these settings, {\em sample efficiency} alone is no longer sufficient to characterize the practicality of an algorithm. {\em Computational efficiency} becomes a critical concern, since naively adopting sample-efficient algorithms under rich function classes often leads to procedures that are computationally intractable. As a result, many existing works rely on strong oracle assumptions, such as access to policy planning oracles \citep{uehara2021representation,cheng2023improved,NEURIPS2020_e894d787} or constrained planning oracles~\citep{modi2021model,mhammedi2023efficient,tan2025actor} that can be prohibitively expensive or even NP-hard in general function approximation settings. For instance, a constrained planning oracle requires computing the policy or action that maximizes a Q-function induced by a model class or confidence set. 
This gap between statistical (i.e., sample-complexity) guarantees and computational feasibility highlights the need for new algorithmic designs that are efficient in both senses.

In this paper, we focus on \emph{low-rank Markov Decision Processes (MDPs)}~\citep{NEURIPS2020_e894d787}, a structured yet expressive class of environments that has received significant attention in recent years. Low-rank MDPs assume that the transition kernel admits a low-rank factorization, meaning that the transition density can be represented using finite-dimensional embeddings of the state--action pair and the next state. This structural assumption captures a broad range of practical environments while enabling refined algorithmic and theoretical analysis. Our goal is to design an RL algorithm for low-rank MDPs that is \emph{provably efficient} in both sample complexity and computational complexity.

A key challenge in formalizing computational efficiency in RL theory is that traditional complexity measures, such as counting arithmetic operations, are often inadequate in function approximation settings. Instead, we adopt a widely used and practically meaningful proxy: \emph{oracle complexity}. In particular, we take \emph{standard supervised learning (regression)} as our basic computational primitive. Supervised learning represents one of the weakest and most commonly accepted oracle assumptions in modern machine learning. Our objective is therefore to design an RL algorithm whose computation can be entirely reduced to supervised learning subroutines. To justify this choice, we analyze the computational structure of commonly used RL oracles and establish a clear hierarchy among them, showing that planning-based oracles require solving strictly more complex optimization problems than policy evaluation, which in turn admits an efficient reduction to regression under low-rank MDPs.

This leads to the central question of this paper:
\begin{center}
    {\it Can we design a provably sample-efficient and computationally efficient algorithm for low-rank MDPs that relies only on supervised learning oracles?} 
\end{center}

We provide affirmative answer to the above question and make the following contributions.

\begin{itemize}[leftmargin=0pt]
    \item \textbf{Oracle Complexity.}
    We study three commonly adopted oracles in the RL literature: \emph{policy evaluation}, \emph{policy planning}, and \emph{constrained planning}. We characterize their objective functions under low-rank MDPs and analyze key regularity properties, including convexity and smoothness. By establishing reduction relationships among these oracles, we demonstrate a hierarchy of best-known oracle complexity upper bounds, showing that planning-based oracles are computationally more demanding than policy evaluation under the best-known reductions, while policy evaluation can be efficiently implemented via supervised learning.

    \item \textbf{Algorithm.}
    We propose a new \emph{optimistic actor--critic} (Opt-AC) algorithm for low-rank MDPs that relies solely on a truncated policy evaluation oracle and therefore can be implemented entirely using supervised learning oracle. In each iteration, our algorithm fixes a learned model and an exploration bonus, constructs an optimistic estimate of the Q-function truncated by the remaining horizon, and updates the policy from step $H$ back to step $1$ within that model, avoiding explicit planning or constrained optimization steps.

    \item \textbf{Sample Complexity.}
    We prove that Opt-AC achieves a real-environment sample complexity of
\(
\tilde{O}\!\left(H^5 d^3 |\mathcal{A}|^2\log(|\Theta|/\delta) / \epsilon^2\right)
\)
for finding an $\epsilon$-optimal policy with probability at least $1-\delta$, using at most $H-1$ truncated policy evaluation oracle calls per iteration. This improves upon the best known result by a factor of $d$. The improvement is enabled by a refined error analysis that directly controls the empirical Gram matrix rather than relying on bounds derived from its expectation.

    \item \textbf{Bridging Theory and Practice.}
    We further show that a class of practical RL environments admit \emph{approximately low-rank} transition structures, and implement a practical variant of Opt-AC in several standard Gym environments. Experimental results demonstrate that this variant consistently outperforms existing baselines across most tasks.
\end{itemize}

\begin{table*}[h]
\centering
\small
\begin{tabular}{lccccc}

\textbf{Name} & \textbf{Oracle Types}
& \textbf{Convexity}
& \textbf{Smoothness}
& \textbf{Computation}
& \textbf{Sample Complexity}  \\

\textbf{FLAMBE} \citep{NEURIPS2020_e894d787}
& $\mathsf{PP}(\varepsilon)$ & \ding{55} & \ding{55}
& $\mathsf{SL}^H(\varepsilon^2|\mathcal{A}|^{-2H})$
& $\tilde O\!\left(\dfrac{H^{22}d^7|\mathcal{A}|^9\log|\Theta|}{\epsilon^{10}}\right)$ 
 \\

\textbf{Rep-UCB} \citep{uehara2021representation}
& $\mathsf{PP}(\varepsilon)$ & \ding{55} & \ding{55}
& $\mathsf{SL}^H(\varepsilon^2|\mathcal{A}|^{-2H})$
& $\tilde O\!\left(\dfrac{H^7d^4|\mathcal{A}|^2\log|\Theta|}{\epsilon^2}\right)$  \\

\textbf{RAFFLE} \citep{cheng2023improved}
& $\mathsf{PP}(\varepsilon)$ &\ding{55} & \ding{55}
& $\mathsf{SL}^H(\varepsilon^2|\mathcal{A}|^{-2H})$
& $\tilde O\!\left(\dfrac{H^5d^4|\mathcal{A}|^2\log|\Theta|}{\epsilon^2}\right)$ 
\\

\textbf{SpanRL} \citep{mhammedi2023efficient} & $\mathsf{CP}(\varepsilon)$  & \ding{55} & \ding{55}  & $\mathsf{SL}^{2^{\tilde O(d^2)}}\!\left(\varepsilon^2 |\mathcal A|^{-2H}\right).$ & $ \tilde{O}\left(\dfrac{H^6d^9|\mathcal{A}|^4\log|\Phi|}{\epsilon^2} \right)$ \\

\textbf{MOFFLE}  \citep{modi2021model} & $\mathsf{CP}(\varepsilon)$ & \ding{55} & \ding{55} & $\mathsf{SL}^{2^{\tilde O(d^2)}}\!\left(\varepsilon^2 |\mathcal A|^{-2H}\right).$ & $ \tilde{O}\left(\dfrac{H^{7}d^{11}|\mathcal{A}|^{14}\log|\Phi|}{\epsilon^2} \right)$ \\

\textcolor{blue}{\textbf{Opt-AC}  (Ours)}
& \textcolor{blue}{$\mathsf{TPE}(\varepsilon)$} & \textcolor{blue}{\checkmark} & \textcolor{blue}{\checkmark}
& \textcolor{blue}{$\mathsf{SL}^{H-1}(\varepsilon^2/H^2)$}
& \textcolor{blue}{$\tilde O\!\left(\dfrac{H^5d^3|\mathcal{A}|^2\log|\Theta|}{\epsilon^2}\right)$}\\

\textbf{Lower Bound} \citep{cheng2023improved} & - & -
& -
& -
& $\Omega\!\left(\dfrac{HdA}{\epsilon^2}\right)$ 

\end{tabular}

\caption{Comparison of algorithms for low-rank MDPs in terms of oracle structure, computational properties, and sample complexity.
\emph{Oracle Types} indicate the highest-level optimization oracle required
(policy evaluation $\mathsf{PE}$, truncated policy evaluation $\mathsf{TPE}$, policy planning $\mathsf{PP}$, or constrained planning $\mathsf{CP}$).
\emph{Convexity} and \emph{Smoothness} refer to the regularity of the objective optimized by the corresponding oracle.
\emph{Computation} reports the best-known $\mathsf{SL}$-oracle complexity required to implement the oracle, expressed using the summary $\mathrm{SL}^{m}(\varepsilon)$ (\Cref{def:sl-aggregate}).
For clarity, coverage constants are omitted.
All results are stated for low-rank MDPs with unknown features.
Here $d$ is the feature dimension, $|\mathcal A|$ is the action space size, and $H$ is the horizon.
The accuracy parameter $\varepsilon$ in the computation column corresponds (up to constants and logarithmic factors) to the inverse square root of the sample complexity.
} 
\label{tab:low-rank-vs-general-fa}
\end{table*}

\section{Related Work}

There is a long line of work on the theoretical foundations of RL~\citep{auer2008near,azar2017minimax,jin2018q,jiang2017contextual,bai2019provably,jin2020provably,menard2021fast,uehara2021representation,cheng2023improved,zhang2022efficient}. Many existing algorithms achieve favorable sample complexity and computational efficiency simultaneously in structured settings such as \emph{tabular} and \emph{linear} MDPs. In these regimes, computational complexity is typically measured by the number of arithmetic operations. For example, algorithms for tabular MDPs often incur computational complexity on the order of $O(H|\mathcal{S}||\mathcal{A}|K)$~\citep{menard2021fast}, while algorithms for linear MDPs typically scale as $O(Hd|\mathcal{A}|K)$~\citep{jin2020provably}. In particular, \citet{lin2025optimistic} designed practical and computationally efficient algorithms through randomized gradient descent for linear MDPs.  However, when extending beyond linear structure to \emph{low-rank MDPs} and more general function approximation settings, existing works often fail to explicitly characterize computational complexity, or rely on algorithmic primitives that are difficult to implement in practice.

\textbf{Low-Rank MDPs.}
Low-rank MDPs were first studied by \citet{NEURIPS2020_e894d787}, who proposed a reward-free algorithm aimed at learning transition dynamics that enable planning under arbitrary reward functions. Subsequently, \citet{DBLP:conf/iclr/UeharaZS22} proposed the REP-UCB algorithm for reward-known low-rank MDPs and achieved a sample complexity of
\(
\tilde O\!\left(\frac{H^7 d^4 |\mathcal{A}|^2 \log |\Theta|}{\varepsilon^2}\right),
\) where $H$ is the episode horizon, $d$ is the feature dimension, $|\mathcal{A}|$ is the numebr of actions, and $\Theta$ is the parameter space.
Later, \citet{cheng2023improved} improved this bound to
\(\tilde O\!\left(\frac{H^5 d^4 |\mathcal{A}|^2 \log |\Theta|}{\varepsilon^2}\right)
\)
in the reward-free setting by introducing a clipping technique. \citet{DBLP:journals/corr/abs-2106-11935} proposed ReLEX for a closely related low-rank model. Despite these improved sample complexity guarantees, the above approaches rely heavily on \emph{policy planning oracles}, which require solving difficult optimization problems and are usually training-unstable in practice. Several subsequent works~\citep{ren2022free,zhang2022making} attempted to bridge theory and practice by exploiting the representational structure of low-rank MDPs and redesigning the representation learning component. Nevertheless, these methods still depend on policy planning oracles, leading to a mismatch between theoretical assumptions and practical implementations. Model-free algorithms have also been proposed for low-rank MDPs~\citep{modi2021model,mhammedi2023efficient}, but these methods rely on even stronger \emph{constrained planning oracles}. From a computational perspective, \citet{rohatgi2025necessary} analyze the necessary oracle (also use regression as a proxy) for block MDPs, a special case of low-rank MDPs, and identify fundamental computational hardness results. Their findings complement our work by presenting a computational limits of low-rank MDPs. 

\textbf{General Function Approximation.}
Beyond low-rank MDPs, a variety of structural assumptions have been proposed to enable sample-efficient RL with general function approximation, including low Bellman or witness rank~\citep{jiang2017contextual,sun2019model}, bilinear function classes~\citep{du2021bilinear}, and low Bellman eluder dimension~\citep{jin2021bellman}. These frameworks subsume low-rank MDPs as special cases and often yield sharp sample complexity guarantees. However, algorithms in these settings are typically computationally intractable~\citep{DBLP:conf/iclr/UeharaZS22}. Another closely related work is \citet{tan2025actor}, which developed an actor--critic algorithm for general function approximation with low Bellman eluder dimension and achieved a sample complexity of
\(
\tilde O\!\left(\frac{dH^5 + dH^4 \log |\Theta|}{\varepsilon^2}\right).
\)
However, this approach also relies on the strongest form of optimization oracle, namely constrained planning.

\section{Problem Setting}

\textbf{Notation.} For any integer $H \in \mathbb{N}$, let $[H] := \{1, \ldots, H\}$. For a vector $x \in \mathbb{R}^d$, $\|x\|_2$ denotes the Euclidean norm. For a positive semidefinite matrix $A \in \mathbb{R}^{d \times d}$, we define $\|x\|_A := \sqrt{x^\top A x}$.   For two probability measures $P$ and $Q$, $\mathtt{D}_{\mathtt{KL}}(P\|Q)$ denotes their KL-divergence. Throughout, we use $\widetilde{O}(\cdot)$ to hide logarithmic factors.

\subsection{RL Formulation}\label{sec:rl formulation}

\textbf{Episodic Markov Decision Processes.} We consider a finite-horizon episodic Markov decision process (MDP)
\(
\mathcal{M} = (\mathcal{S}, \mathcal{A}, H, \{\mathcal{T}_{\theta,h}\}_{h=1}^H, \{r_h\}_{h=1}^H),
\)
where $\mathcal{S}$ is a (possibly infinite) state space, $\mathcal{A}$ is a finite action space,
$H \in \mathbb{N}$ is the episode length, $\mathcal{T}_{\theta,h}(\cdot \mid s,a)$ is a time-dependent transition kernel
parameterized by $\theta \in \Theta$, and $r_h : \mathcal{S} \times \mathcal{A} \to [0,1]$ is a deterministic known reward function.
Each episode starts from a fixed initial state $s_1$.  

A (non-stationary) policy $\pi = \{\pi_h\}_{h=1}^H$ consists of mappings $\pi_h : \mathcal{S} \to \Delta(\mathcal{A})$.
At step $h$, the agent observes $s_h$, selects $a_h \sim \pi_h(\cdot \mid s_h)$,
receives reward $r_h(s_h,a_h)$, and transitions to $s_{h+1} \sim \mathcal{T}_{\theta,h}(\cdot \mid s_h,a_h)$.

For a policy $\pi$ and model parameter $\theta$, the value and action-value functions are defined as
\begin{align*}
V^{\pi}_{\theta,r,h}(s)
&:= \mathbb{E}_{\pi,\theta}\!\left[ \sum_{t=h}^H r_t(s_t,a_t) \,\middle|\, s_h = s \right], \\
Q^{\pi}_{\theta,r,h}(s,a)
&:= r_h(s,a) + \mathbb{E}_{s' \sim \mathcal{T}_{\theta,h}(\cdot \mid s,a)}\!\left[ V^{\pi}_{h+1,\theta}(s') \right],
\end{align*}
with terminal condition $V^{\pi}_{\theta,r, H+1}(\cdot) \equiv 0$.
We write $V^{\pi}_{\theta} := V^{\pi}_{\theta,r,1}(s_1)$.
Let $\theta^\star$ denote the true model parameter, and let
\(
V^\star := \sup_{\pi} V^{\pi}_{\theta^\star}
\)
be the optimal value. Our goal is to learn an $\epsilon$-optimal policy $\bar{\pi}$ such that $V^* - V_{\theta^*,r}^{\bar{\pi}}\leq \epsilon$.

Throughout, expectations $\mathbb{E}_{\theta}^{\pi}[\cdot]$ are taken with respect to the trajectory distribution
induced by policy $\pi$ and transition model $\mathcal{T}_\theta$.

\paragraph{Low-Rank MDPs.}
We study episodic MDPs whose transition dynamics admit an exact low-rank factorization with unknown feature representations. 
For each step $h \in [H]$, the transition kernel $\mathcal{T}_{\theta,h}$ is said to be \emph{low-rank with dimension $d$} if there exist latent feature maps
$\phi_{\theta,h} : \mathcal{S} \times \mathcal{A} \to \mathbb{R}^d$ and
$\mu_{\theta,h} : \mathcal{S} \to \mathbb{R}^d$ such that $\forall (s,a,s') \in \mathcal{S} \times \mathcal{A} \times \mathcal{S}$, we have 
\[
\mathcal{T}_{\theta,h}(s' \mid s,a) = \langle \phi_{\theta,h}(s,a), \mu_{\theta,h}(s') \rangle.
\]
We assume the normalization conditions $\|\phi_{\theta,h}(s,a)\|_2 \le 1$ and
$\bigl\|\int \mu_{\theta,h}(s') g(s') \, ds'\bigr\|_2 \le \sqrt{d}$ for all measurable $g:\mathcal{S}\to[0,1]$.
An MDP is called a low-rank MDP if this factorization holds for all $h \in [H]$.

The feature representations $\{\phi_{\theta,h},\mu_{\theta,h}\}$ are latent and unknown to the agent, so learning the transition model requires learning these representations from data. We adopt a realizability assumption: the true transition and reward models belong to known function classes, which is standard in low-rank and general function approximation settings~\citep{jin2020provably,uehara2021representation,jin2021bellman}.

\begin{assumption}[Realizability]\label{assumption: realizability}
The agent is given function classes $\Phi$ and $\Psi$ such that
\(
\phi_{\theta^\star,h} \in \Phi,\mu_{\theta^\star,h} \in \Psi
, \forall h \in [H].
\)
\end{assumption}

For simplicity of exposition, we assume the model class
\(
\Theta := \Phi \times \Psi
\)
is finite.
This is used to simplify confidence bounds and concentration arguments; extensions to infinite classes with bounded statistical complexity follow standard covering number techniques and are omitted.

\subsection{Optimization Oracles and Computational Model}

To formalize computational efficiency, we measure complexity through access to
\emph{optimization oracles}, a standard abstraction in theoretical reinforcement learning.
Each oracle is specified by the functional objective it must solve and the accuracy
criterion under which success is measured. We focus on four oracles commonly adopted
in the RL literature, together with a truncated variant of policy evaluation used by our algorithm, and characterize their complexity solely through
$\mathsf{SL}$-oracle complexity.

\begin{definition}[Supervised Learning Oracle]
A supervised learning oracle $\mathsf{SL}(\varepsilon)$ takes as input a parameter space $\mathcal W$, a data distribution (sampling oracle) $\rho$ over input-output pairs $(x,y)$, and a loss function $\ell:\mathbb R\times\mathcal Y\to\mathbb R_+$.
It returns a parameter $\hat w\in\mathcal W$ such that
\[
\mathbb E_{(x,y)\sim\rho}\!\left[\ell\!\left(f_{\hat w}(x),y\right)\right]
\le
\inf_{w\in\mathcal W}
\mathbb E_{(x,y)\sim\rho}\!\left[\ell\!\left(f_w(x),y\right)\right]
+\varepsilon.
\]
\end{definition}
{We emphasize that $\mathsf{SL}(\varepsilon)$ models a single-level risk minimization problem with a fixed hypothesis class $\mathcal W$ and a simple loss function $\ell$ (e.g. L2 or cross-entropy loss); in particular, it does not include constrained or bilevel optimization where the hypothesis class or loss depends on data, input or another optimization problem.} 


\begin{definition}[Policy Evaluation Oracle]
A policy evaluation oracle $\mathsf{PE}(\varepsilon)$ takes as input
a model $\theta$, a policy $\pi$, and a reward function $r$, and returns
an action-value function $\{\hat{Q}_{h}\}_{h=1}^H$ such that, for a
specified distribution $\rho$,
\[
\mathbb E_{(s,a)\sim\rho}
\!\left[\,\big|Q^{\pi}_{\theta,r,h}(s,a)-\hat Q_{h}(s,a)\big|\,\right]
\le \varepsilon
\quad \text{for all } h.
\]
\end{definition}

\begin{definition}[Truncated Policy Evaluation Oracle]
\label{def:tpe}
Given a model $\theta$, a Markov policy $\pi$, nonnegative stage rewards $g=(g_h)_{h=1}^H$, and caps $U=(U_h)_{h=1}^H$, set $\bar V^{\pi}_{\theta,g,H+1}\equiv 0$ and define
\begin{align*}
\bar Q^{\pi}_{\theta,g,h}(s,a)
&=
\min\Big\{U_h,\; g_h(s,a)+\mathbb E_{s'\sim\mathcal T_{\theta,h}(\cdot\mid s,a)}\big[\bar V^{\pi}_{\theta,g,h+1}(s')\big]\Big\},\\
\bar V^{\pi}_{\theta,g,h}(s)
&=
\sum_{a\in\mathcal A}\pi_h(a\mid s)\,\bar Q^{\pi}_{\theta,g,h}(s,a).
\end{align*}
A truncated policy evaluation oracle $\mathsf{TPE}(\varepsilon)$ takes as input
a model $\theta$, a policy $\pi$, stage rewards $g$, caps $U$, and a queried stage $h$, and returns
an action-score function $\hat Q_h$ such that, for a
specified distribution $\rho$ on state--action pairs,
\[
\mathbb E_{(s,a)\sim\rho}
\!\left[\,\big|\hat Q_{h}(s,a)-\bar Q^{\pi}_{\theta,g,h}(s,a)\big|\,\right]
\le \varepsilon.
\]
A randomized oracle provides this guarantee with the requested confidence conditional on the query and its history.
\end{definition}
The truncation acts on the conditional expectation plus the immediate reward at every step, and the truncated $\bar V^{\pi}_{\theta,g,h}$ enters the recursion of the previous step. Taking $U_h=+\infty$ for all $h$ recovers the value target of $\mathsf{PE}$. As in $\mathsf{PE}$, the error is measured in $L^1(\rho)$. A stagewise output suffices for our purpose, since the backward actor update in \Cref{sec:algorithm} only needs the scores of the actions at the current stage.

\begin{definition}[Policy Planning Oracle]
A policy planning oracle $\mathsf{PP}(\varepsilon)$ takes as input
a model $\theta$ and reward function $r$, and returns an action-value
function $\{\hat{Q}_{h}\}_{h=1}^H$ satisfying, for a specified
distribution $\rho$,
\[
\mathbb E_{(s,a)\sim\rho}
\!\left[\,\big|Q^\star_{\theta,r, h}(s,a)-\hat Q_{h}(s,a)\big|\,\right]
\le \varepsilon
\quad \text{for all } h.
\]
\end{definition}

\begin{definition}[Constrained Planning Oracle]
A constrained planning oracle $\mathsf{CP}(\varepsilon)$ is given
a reward function $r$ and a constraint set over models
\(
\Theta_c := \{\theta : L(\theta)\le c\},
\)
where $L(\theta)$ is a supervised learning loss.
It returns an action-value function $\hat Q$ such that, for a specified
distribution $\rho$,
\[
\mathbb E_{(s,a)\sim\rho}
\!\left[
\big|\hat Q_h(s,a)-\sup_{\theta\in\Theta_c} Q^\star_{\theta,h}(s,a)\big|
\right]
\le \varepsilon
\quad \text{for all } h.
\]
\end{definition}

\noindent\textbf{SL-oracle Complexity.}
We measure computational efficiency through the number and accuracy of calls to supervised learning oracles, treating supervised learning as the basic computational primitive. This abstraction allows us to compare different reinforcement-learning primitives without committing to a specific optimization algorithm or implementation.

\begin{definition}[$\mathsf{SL}$-oracle complexity]
\label{def:sl-aggregate}
Consider an algorithm that invokes supervised learning oracles
$\mathsf{SL}(\varepsilon_1),\ldots,\mathsf{SL}(\varepsilon_m)$.
We define its  $\mathsf{SL}$-oracle complexity by \(
m\mathrm{SL}\!\left(\varepsilon_{\min}\right)\), where \(
\varepsilon_{\min}:=\min_{i\in[m]}\varepsilon_i.
\) This definition records the total number of oracle calls and the most stringent accuracy requirement.
\end{definition}

With \Cref{def:sl-aggregate}, we present the following proposition that builds a hierarchy among $\mathsf{PE}, \mathsf{PP},$ and $\mathsf{CP}$ in terms of $\mathsf{SL}$-oracle-complexity. The proof can be found in \Cref{sec:oracle analysis}.

\begin{proposition}[Best-known $\mathsf{SL}$-oracle complexity]
\label{prop:sl-complexity-summary}
Under low-rank MDP settings, suppose the data distribution $\rho$ has coverage constant $C$ with respect to the uniform policy $\mathtt{u}$, i.e.,
\[
\mathbb E_{(s,a)\sim\rho}
\big[\mathcal T_h(s' \mid s,a)\mathtt{u}(a'\mid s')\big]
\le
C\,\rho(s',a').
\]
Then the following best-known $\mathsf{SL}$-oracle complexity upper bounds hold for
achieving an $\varepsilon$-approximate solution.

\begin{enumerate}[leftmargin=*, itemsep=1pt, topsep=2pt]
    \item $\mathsf{PE}(\varepsilon):$
    \(
    \mathsf{SL}\!\left(\tilde O\!\left(\varepsilon^2/H\right)\right).
    \)

    \item $\mathsf{PP}(\varepsilon):$
    \(
    H\mathsf{SL}\!\left(\tilde O\!\left(\varepsilon^2 (|\mathcal A|C)^{-2H}\right)\right).
    \)

    \item $\mathsf{CP}(\varepsilon):$ $2^{\tilde O(d^2)}H\mathsf{SL}^{}\!\left(\tilde O\!\left(\varepsilon^2 (|\mathcal A|C)^{-2H}\right)\right)$, attained by covering enumeration.
\end{enumerate}
Moreover, only policy evaluation's loss function is convex and smooth.
\end{proposition}
The bound for $\mathsf{PE}$ does not require the coverage condition. In addition, a stage-$h$ query of $\mathsf{TPE}(\varepsilon)$ (\Cref{def:tpe}) can be answered with at most $H-1$ calls of $\mathsf{SL}(\varepsilon^2/H^2)$ when the given model can be simulated from $\rho$ (\Cref{prop:tpe-supervised}). All bounds in \Cref{prop:sl-complexity-summary} are upper bounds attained by explicit reductions; in particular, the bound for $\mathsf{CP}$ is an upper bound for the covering-enumeration algorithm in \Cref{sec:oracle analysis}, rather than an oracle lower bound.

{Intuitively, policy evaluation corresponds to a convex and smooth optimization problem, while policy planning is non-smooth due to the Bellman optimality operator and requires a backward solver with $H$ stages. Constrained planning further induces a bilevel optimization structure, for which only stationary-point guarantees are known~\cite{huang2022enhanced,beck2023computationally}, and global optimization is believed to require exponentially many supervised learning oracle calls\footnote{Proposition~\ref{prop:sl-complexity-summary} summarizes a hierarchy of \emph{best-known oracle upper bounds and optimization guarantees}. It does not preclude the existence of more efficient reductions for $\mathsf{PP}$/$\mathsf{CP}$, and establishing oracle lower bounds is an open problem.}.}

\section{Algorithm: Optimistic Actor--Critic with Representation Learning}\label{sec:algorithm}


In this section, we present an optimistic actor–critic algorithm for low-rank MDPs. The algorithm proceeds in episodes. At each iteration, it (i) learns a transition model via maximum likelihood estimation, (ii) constructs an optimistic critic by incorporating an exploration bonus and truncating the Bellman recursion by the remaining horizon, (iii) updates the policy through actor steps from step $H$ back to step $1$ within the learned model, and (iv) collects exploratory data using the updated policy.

Before introducing the components of the algorithm in detail, we summarize
the hyperparameters and oracle access required.
Let $K$ denote the total number of iterations and let $\epsilon>0$ denote the target accuracy, corresponding to an $\epsilon$-optimal policy. We set $\lambda=1/d$, $\beta=c\log(2KH|\Theta|/\delta)$, and $\alpha^2=c_0(\lambda d+|\mathcal A|\beta)$, where $c,c_0$ are sufficiently large absolute constants. We describe the algorithm for $H\ge 2$; when $H=1$, it suffices to choose the best action for the known reward.
Our algorithm relies on two oracles:
(i) a truncated policy evaluation oracle (\Cref{def:tpe}), which can be implemented via supervised learning, and (ii) a sampling oracle $\rho$ over the state-action space. The sampling oracle $\rho$ is required to provide sufficient coverage of the optimal policy's occupancy measure in the learned models. Specifically, letting $d^{\pi}_{\theta,h}(s)$ denote the marginal state distribution induced by the enviroment $\theta$ and a policy $\pi$ at step $h$, we assume that there exists a finite constant $C$ such that
\(
d^{\pi^\star}_{\hat{\theta}_k,h}(s)\,\mathtt{u}(a)
\;\le\;
C\,\rho(s,a),
\forall (s,a),
\)
for all iterations $k$ and steps $h$, where $\mathtt{u}$ denotes the uniform distribution over actions, $\hat{\theta}_k$ is the estimated model at iteration $k$ (defined later). This is a measure-domination condition, and $C$ only enters our guarantee through the oracle error term. Using (a known upper bound on) $C$, we set the oracle accuracy to $\tau=\epsilon/(4C|\mathcal A|H)$. 
This assumption is mild and can be satisfied, for example, by choosing $\rho$ as a uniform distribution over the state-action space when state space is compact~\citep{towers2024gymnasium}.

\paragraph{Difference to Existing Algorithms.}
Our design differs fundamentally from existing approaches for low-rank MDPs in both algorithmic structure and oracle requirements. {\bf (a) Low-computation-cost Oracle.} In contrast to prior methods that rely on policy planning or constrained planning oracles, our algorithm requires only a \emph{truncated policy evaluation} oracle, which can be efficiently implemented via supervised learning (\Cref{prop:tpe-supervised}). {\bf (b) Approximate Oracle Solver.} Moreover, we do not assume access to exact oracle solutions. Instead, our analysis accommodates a more realistic \emph{sample-then-optimize} setting, in which all optimization steps can utilize finite data and approximate solvers. {\bf (c) Actor-critic Style Design.} Our method departs from pure UCB-style designs and adopts an optimistic actor-critic framework. The optimistic critic is constructed directly from the learned model without invoking any global planning subroutine, and each actor step admits an exact-implementable method within the fixed learned model. As a result, every component of the algorithm is implementable using standard supervised learning and sampling oracles.
The pseudocode is presented in Algorithm~\ref{alg:main}.

\paragraph{Collect Exploratory Data.} Suppose we are at iteration $k$ and have constructed a policy $\pi^{(k)}$. For each $j=1,\ldots,H-1$, start an independent episode in the true environment, execute policy $\pi^{(k)}$ for the first $j-1$ steps, then execute uniform policy $\mathtt{u}$ for step $j, j+1$, and record the window $X_{k,j}=(s_j^{k,j},a_j^{k,j})$ and $Y_{k,j}=(s_{j+1}^{k,j}, a_{j+1}^{k,j}, s_{j+2}^{k,j})$. When $j=H-1$, $s_{H+1}^{k,j}$ is generated by the known dummy terminal transition. The idea is that for episode $(k,j)$, $a_j$ and $a_{j+1}$ are selected uniformly randomly. The second uniform exploration facilitates the transition estimation at step $j+1$ and thus the feature extraction, while the first uniform exploration helps to define the uncertainty at step $j$. The exploration bonus only enters the evaluation in the learned model.

\paragraph{Representation Learning via Maximum Likelihood.}

To learn the latent transition structure, the agent maintains the windows
$\{(X_{n,j},Y_{n,j})\}_{n<k,\,j<H}$ collected before iteration $k$.
Given the collected data, we estimate the model parameter
$\theta \in \Theta$ by minimizing the negative log-likelihood of the two-step transitions in each window jointly:
\begin{equation}
\label{eq:mle}
\mathcal{L}^{(k)}(\theta)
= - \sum_{n < k}\sum_{j=1}^{H-1}
\log p_{\theta}(Y_{n,j} \mid X_{n,j}),
\end{equation}
where $p_{\theta}(Y_{n,j} \mid X_{n,j})=\mathcal{T}_{\theta,j}(s_{j+1}^{n,j} \mid X_{n,j})\,\mathtt{u}(a_{j+1}^{n,j})\,\mathcal{T}_{\theta,j+1}(s_{j+2}^{n,j} \mid s_{j+1}^{n,j}, a_{j+1}^{n,j})$ is the conditional likelihood of the window under the candidate model $\theta$.

Rather than requiring exact minimization, we allow approximate empirical risk minimization. Specifically, the estimator $\hat{\theta}_k$ satisfies
\(
\mathcal{L}^{(k)}(\hat{\theta}_k)
\le \min_{\theta \in \Theta} \mathcal{L}^{(k)}(\theta) + \beta,
\)
where $\beta$ is a confidence width that scales logarithmically with the model class size and the total number of samples. In the first iteration, $\hat{\theta}_1$ can be any candidate model. Note that this step only requires one supervised learning oracle $\mathsf{SL}$ call; if the oracle is called with the average loss, the additive tolerance $\beta$ is divided by the number of windows. If $\Theta$ is a stagewise model class, the observations of each transition at its two window positions can be pooled stage by stage (see \Cref{sec:general MLE analysis}).

\paragraph{Optimistic Critic via Exploration Bonus.}

To encourage efficient exploration, we augment the reward with an optimism-driven exploration bonus. Specifically, for each step $h<H$, we define the empirical Gram matrix using the window starts
\begin{equation*}
\hat{\Lambda}_{k,h}
= \lambda I
+ \sum_{n < k}
\hat{\phi}_{h}^k(X_{n,h})
\hat{\phi}_{h}^k(X_{n,h})^\top,
\end{equation*}
where $\hat{\phi}_{h}^k = \phi_{\hat{\theta}_k,h}$ is the learned feature of $\hat{\theta}_k$. The raw exploration bonus is defined as
\begin{equation}
\label{eq:bonus}
\tilde{b}_h^{(k)}(s,a)
= \min \Bigl\{
\alpha \|\phi_{\hat{\theta}_k,h}(s,a)\|_{\hat{\Lambda}_{k,h}^{-1}},\, 1
\Bigr\}.
\end{equation}

The term $\|\hat{\phi}_k(s,a)\|_{\hat{\Lambda}_{k,h}^{-1}}$ captures the statistical uncertainty of the learned representation in direction $(s,a)$, analogous to elliptical confidence bounds in linear bandits.
We set the actual exploration bonus function as $\hat{b}_h^{(k)} = 3H\tilde{b}_h^{(k)}$ for $h<H$, and $\tilde{b}_H^{(k)}=\hat{b}_H^{(k)}=0$.
The multiplicative scaling is chosen so that the bonus term dominates the propagation of one-step transition estimation error through the $H$-step Bellman recursion.
As a result, with high probability, the truncated critic constructed using $r + \hat{b}^{(k)}$ can fall below the true action value only by the transition estimation error at the current step (\Cref{lemma:truncated-value-bounds}).
Throughout iteration $k$, $(\hat{\theta}_k,\hat{b}^{(k)})$ is kept fixed while the policy is constructed.

Using the learned model $\hat{\theta}_k$, the augmented reward $r+\hat{b}^{(k)}$, and the caps $U_h=H-h+1$, we write $\bar Q^{\pi}_{k,h}:=\bar Q^{\pi}_{\hat{\theta}_k, r+\hat{b}^{(k)},h}$ for the truncated critic in \Cref{def:tpe}. When step $h<H$ is processed, the policies $\pi^{(k)}_{h+1:H}$ have already been determined (the policies before step $h$ can be completed arbitrarily, since they do not affect the target), and the critic \(\bar Q^{\pi^{(k)}}_{k,h}(s,a)\) is estimated by calling the truncated policy evaluation oracle $\mathsf{TPE}(\tau)$ at stage $h$. So we obtain $\hat{Q}_{k,h}$ that satisfies
\begin{equation}\label{eq:critic-estimate}
\Eb_{(s,a)\sim\rho}\left[ \left|\hat{Q}_{k,h}(s,a) - \bar Q^{\pi^{(k)}}_{k,h}(s,a) \right| \right] \leq \tau.
\end{equation}

\if{0}
Note that the exact formulation is
\begin{align}
\label{eq:critic}
& Q^{\pi_k}_{\hat{\theta}_k, r + \hat{b}^{(k)}, h}(s,a)
 = r_h(s,a) + \hat{b}_h^{(k)}(s,a)
\nonumber \\
&\hspace{1.5cm} + \mathbb{E}_{s' \sim \mathcal{T}_{\hat{\theta}_k,h}(\cdot \mid s,a)}
\!\left[ V^{\pi_k}_{\hat{\theta}_k, r + \hat{b}^{(k)},h+1}(s') \right].
\end{align}
\fi

\paragraph{Actor Update via Backward Policy Improvement.}

Given the optimistic critic, the policy is updated from step $H$ back to step $1$ within the fixed learned model $\hat{\theta}_k$. At the terminal step, we choose $\pi^{(k)}_H(s)\in\arg\max_a r_H(s,a)$. For $h=H-1,\ldots,1$, the policy at step $h$ is updated via an actor step on $\hat{Q}_{k,h}$.
For the uniform reference policy $\mathtt{u}$ and learning rate $\eta>0$,
we define the actor objective
\begin{align}
\label{eq:actor-objective}
\mathcal{L}_{\mathrm{a}}^{(k)}(\pi)
&=
\left\langle \pi_h(\cdot \mid s), \hat{Q}_{k,h}(s,\cdot) \right\rangle \nonumber \\
&\quad
-
\frac{1}{\eta}
\mathtt{D}_{\mathtt{KL}}\!\left(
\pi_h(\cdot \mid s)\,\|\,\mathtt{u}
\right).
\end{align}
This objective is strictly concave in $\pi_h(\cdot\mid s)$ and admits the
closed-form solution $\pi_h(a\mid s)\propto\exp(\eta\,\hat{Q}_{k,h}(s,a))$.
For clarity of presentation and analysis, we use its greedy limit $\eta\to\infty$:
\begin{equation}
\label{eq:policy-update}
\pi_h^{(k)}(s)
\in
\arg\max_{a\in\mathcal{A}}
\hat{Q}_{k,h}(s,a),
\end{equation}
where ties are broken by a fixed rule. The determined policies $\pi^{(k)}_{h:H}$ are then used in the query of step $h-1$.
With a finite learning rate, the softmax solution loses at most $\log|\mathcal{A}|/\eta$ per step relative to the greedy action; choosing $\tau=\epsilon/(8C|\mathcal{A}|H)$ and $\eta\ge 4H\log|\mathcal{A}|/\epsilon$ leaves half of the error budget for the model error, and the number of real trajectories has the same order.

This update can be implemented efficiently in polynomial time.
In particular, it suffices to store the estimated critics
$\{\hat{Q}_{k,h}\}_{h}$ of the current iteration, and at interaction time to select
$a_h$ greedily with respect to $\hat{Q}_{k,h}(s_h,\cdot)$.
After $K$ iterations, the algorithm returns $\bar{\pi}=\mathrm{Uniform}\{\pi^{(1)},\ldots,\pi^{(K)}\}$: at the beginning of an episode, one of the $K$ policies is drawn uniformly and executed until termination, so that $V^{\bar{\pi}}_{\theta^\star}=K^{-1}\sum_{k}V^{\pi^{(k)}}_{\theta^\star}$.
Importantly, the actor update does not require solving a global planning
problem, which distinguishes our approach from prior UCB-style algorithms
that rely on policy or constrained planning oracles. Moreover, since the policy improvement is completed within each fixed learned model, the actor error does not accumulate across iterations.

\begin{algorithm}[t]
\caption{Optimistic Actor--Critic (Opt-AC)}
\label{alg:main}
\begin{algorithmic}[1]
\STATE Set hyperparameters: maximum iteration number $K$, error rate $\epsilon$, oracle accuracy $\tau$, coefficients $\beta,\alpha,\lambda$.
\FOR{$k = 1,\ldots,K$}
\STATE Learn model $\hat{\theta}_k$ via approximate MLE (\ref{eq:mle}) on windows $\{(X_{n,j},Y_{n,j})\}_{n<k,j<H}$. $\textcolor{blue}{(\mathsf{SL})}$
\STATE Construct exploration bonus $\hat{b}^{(k)}$ via (\ref{eq:bonus}).
\STATE Set $\pi^{(k)}_H(s)\in\arg\max_a r_H(s,a)$.
\FOR{$h = H-1,\ldots,1$}
\STATE Estimate truncated optimistic critic $\bar Q^{\pi^{(k)}}_{k,h}$ by (\ref{eq:critic-estimate}). $\textcolor{blue}{(\mathsf{TPE})}$
\STATE Update policy $\pi^{(k)}_h$ by (\ref{eq:policy-update}).
\ENDFOR
\STATE Collect $H-1$ exploratory windows $\{(X_{k,j},Y_{k,j})\}_{j=1}^{H-1}$ using $\pi^{(k)}$.
\ENDFOR
\STATE Return policy $\bar{\pi} = \mathrm{Uniform}\{\pi^{(1)},\ldots,\pi^{(K)}\}$
\end{algorithmic}
\end{algorithm}

\section{Theoretical Results}\label{sec:theory}
In this section, we present our main results. 

\begin{theorem}[Computational Complexity]\label{thm:computation}
Running $K$ iterations, \Cref{alg:main} performs $K$ joint likelihood fits for model learning and at most $(H-1)K$ calls of $\mathsf{TPE}(\tau)$, each with $L^1(\rho)$ accuracy $\tau=\epsilon/(4C|\mathcal{A}|H)$. If $\mathsf{TPE}$ is implemented via supervised learning as in \Cref{prop:tpe-supervised}, the policy construction uses at most $KH(H-1)/2$ squared-loss regressions, each with excess risk $\xi=\tau^2/H^2$, i.e., $\frac{KH(H-1)}{2}\mathsf{SL}^{}(\tau^2/H^2)$, in addition to model learning. Hence the total number of supervised learning oracle calls is $O(KH^2)$.
\end{theorem}

\begin{proof}
We analyze the computational cost of the proposed algorithm and show that it remains lightweight.
The main computational overhead arises in Steps 3--8 of Algorithm~1. Step~3 (model learning) is a standard supervised learning problem with the log loss, and it is called once per iteration; if the solver is called separately for each stage, the actual number of calls should be counted instead. Step~4 constructs the empirical Gram matrices using the learned feature representations; this step requires only a single pass over the collected data and does not involve any iterative optimization, incurring only additional memory to store the matrices. Step~7 performs truncated policy evaluation at stages $h=H-1,\ldots,1$, which uses at most $H-1$ oracle calls per iteration. By \Cref{prop:tpe-supervised}, a query at stage $h$ can be answered with $H-h$ regressions with excess risk $\xi=\tau^2/H^2$, so one iteration uses $\sum_{h=1}^{H-1}(H-h)=H(H-1)/2$ regressions; the simulated samples from the learned model and the solving time are counted separately. Step~8 is a closed-form update. Since the total iteration of \Cref{alg:main} is $K$, the claimed counts follow.
\end{proof}

Then, we state our main theorem for sample complexity. Proofs can be found in \Cref{sec:proof of main thm}.

\begin{theorem}[Sample Complexity]\label{thm:main}
Assume $\Mc$ is a low-rank MDP with dimension $d$, and Assumption \ref{assumption: realizability} holds. Suppose further that
(i) the initial state $s_1$ is fixed, the reward $r_h(s,a)\in[0,1]$ is known, the action set is finite, the true model $\theta^\star$ belongs to the finite model class $\Theta$, and the transition after step $H$ is a known dummy terminal transition;
(ii) every candidate model $\theta\in\Theta$ is a valid probability kernel with the normalized low-rank factorization in \Cref{sec:rl formulation}; and
(iii) for all iterations $k$ and steps $h$, $d^{\pi^\star}_{\hat{\theta}_k,h}\times\mathtt{u}\le C\rho$.
Given any $\epsilon,\delta \in (0,1)$, set $\lambda=1/d$, $\beta=c\log(2KH|\Theta|/\delta)$, $\alpha^2=c_0(\lambda d+|\mathcal{A}|\beta)$, and $\tau=\epsilon/(4C|\mathcal{A}|H)$, where $c,c_0$ are sufficiently large absolute constants, let each call of a randomized $\mathsf{TPE}$ oracle fail with probability at most $\delta/(2KH)$, and choose $K=\tilde{O}(\frac{H^4d^3|\mathcal{A}|^2\log(2|\Theta|/\delta)}{\epsilon^2})$ with a sufficiently large constant. Let $\bar{\pi}=\mathrm{Uniform}\{\pi^{(1)},\ldots,\pi^{(K)}\}$ be the output of Opt-AC
and $\pi^\star$ be the optimal policy under the true model $P^\star$. Then, with probability at least $1-\delta$, we have $V^{\pi^\star}_{P^\star,r}-V^{\bar{\pi}}_{P^\star,r}\leq \epsilon$, and the total number of trajectories collected by Opt-AC in the true environment is $(H-1)K$, which is upper bounded by
$\tilde{O}(\frac{H^5d^3|\mathcal{A}|^2\log(2|\Theta|/\delta)}{\epsilon^2})$.
\end{theorem}

When $\mathsf{TPE}$ is used as the base oracle, \Cref{thm:main} holds directly under its accuracy guarantee; to further claim that $\mathsf{TPE}$ is implemented via supervised learning, \Cref{prop:tpe-supervised} additionally requires the ability to simulate the learned model from $\rho$. The guarantee in \Cref{thm:main} is established for the backward update in \Cref{alg:main}; a bound of the same order for a single cross-iteration actor update, such as the mirror-descent update $\pi^{(k+1)}_h\propto\pi^{(k)}_h\exp(\eta\hat{Q}_{k,h})$, remains to be proven.

Compared with RAFFLE~\cite{cheng2023improved}, our result improves their sample complexity of $H^5d^4|\mathcal{A}|^2$ by a factor of $d$. Note that $d^3$ is also a common order in linear MDP setting where the feature is known~\citep{jin2020provably}. The improvement is brought by a more careful analysis to the one-step back inequality. One-step back inequality is a standard technique in low-rank MDP literature. It shows that the bonus in the truncated optimistic $Q$-value, $\bar Q^{\pi}_{k,h}$, is indeed an upper bound of the estimation error of $\hat{\theta}_k$ propagated from the next step (\Cref{lemma:key-inequality}). One of the key factors that affect the sample complexity is the coefficient of $\|\phi(s_h,a_h)\|_{\hat{\Lambda}_{k,h}^{-1}}$ in $\hat{b}^{(k)}$. Here our choice is $\tilde{O}(H\sqrt{|\mathcal{A}|})$, while existing works often choose $\tilde{O}(H\sqrt{d^2 + |\mathcal{A}|})$~\citep{uehara2021representation,cheng2023improved}, 
which leads to a higher-order dependence on 
$d$. Our analysis also differs fundamentally from that of \citet{tan2025actor}, who study actor–critic methods using a $\mathsf{CP}$ oracle over a general function class. In contrast, operating in the low-rank MDP setting and \emph{explicit bonus} design requires substantially more delicate control of the evolution of the gap between the optimal value function and the optimistic critic value. This gap directly affects policy updates and can accumulate over time if not carefully bounded, making actor–critic-style analysis significantly more challenging in our setting. We control this gap step by step through the truncated critic and complete the policy improvement within each fixed learned model, so that the summation over iterations only involves the current learned policy under the true occupancy, and $C$ only appears in the oracle error term. Despite these difficulties, our approach achieves the best known sample complexity guarantees for low-rank MDPs while relying only on supervised learning oracles.

\section{Results for Misspecified Setting}
An MDP is $\zeta$-approximate low-rank if there exists $\zeta>0$ and $\theta\in\Theta$ such that for all $(s_h,a_h)$, we have 
\[
\left|\mathcal{T}_{\theta,h}(s_{h+1} \mid s_h,a_h) - \langle \phi_{\theta,h}(s_h,a_h), \mu_{\theta,h}(s_{h+1}) \rangle \right| \leq \zeta.
\]

\paragraph{Richness of Approximate Low-rank MDPs.}
We provide a formal justification showing that the class of
$\zeta$-approximate low-rank MDPs is expressive and encompasses a broad family of MDPs commonly encountered in practice.
 
\noindent\emph{Gaussian transition kernels.}
As a motivating example, consider Gaussian transition kernels
$\mathcal{T}_{\theta, h}(\cdot\mid s,a)=\mathcal{N}(f_h(s,a),I)$.
Such transitions admit approximate low-rank representations via random Fourier features~\citep{rahimi2007random}, where the dependence on $(s,a)$ enters through the mean function
$f_h(s,a)$ and the next-state $s'$ is encoded using a shared Fourier basis.
Standard concentration results imply that the transition density can be
approximated uniformly by an inner product
$\langle\phi_h(s,a),\mu_h(s')\rangle$ with error $O(d^{-1/2})$ on compact domains.
Consequently, choosing $d=O(\zeta^{-2})$ yields a $\zeta$-approximate low-rank
representation, a construction that has been exploited in prior work such as
\citet{ren2022free}.

\medskip
\noindent\emph{Beyond Gaussian transitions.}
While random Fourier features are classically associated with Gaussian kernels, the approximation framework extends to a much broader class of conditional distributions satisfying mild regularity conditions. In particular, we explicitly construct an approximation algorithm cRFF~(\Cref{alg:rff}) and
characterize its error guarantees in \Cref{thm:rff}.
This result complements our theoretical analysis and further bridges the gap between the low-rank MDP abstraction and practical transition models. The proof can be found in \Cref{sec:RBF proof}.

\begin{theorem}[Approximate low-rank representation]
\label{thm:rff}
Let $\mathcal{S}\subset\mathbb{R}^D$ be compact.
For each $(\theta,h)\in\Theta\times[H]$ and $(s,a)\in\mathcal{S}\times\mathcal{A}$,
assume access to a sampling oracle for
$\mathcal{T}_{\theta,h}(\cdot\mid s,a)$ that provides $N$ i.i.d.\ next-state
samples.
Suppose the transition kernels satisfy the following conditions.

\textbf{(A1) Boundedness.}
There exists an $M>0$ such that
$\sup_{\theta,s,a,s',h}\mathcal{T}_{\theta,h}(s'\mid s,a)\le M$.

\textbf{(A2) Smoothness and boundary vanishing conditions.}
There exist an integer $m>D$ and such that
$\mathcal{T}_{\theta,h}(\cdot\mid s,a)$ is $m$-times continuously differentiable
on $\mathcal{S}$, and all derivatives up to order $m-1$ vanish on the boundary
$\partial\mathcal{S}$.

Then, for any $\zeta>0$, there exists a randomized algorithm cRFF (\Cref{alg:rff}) that constructs $2d(\zeta)$-dimensional feature maps $\hat\phi_\theta(s,a), \mu(s')\in\mathbb{R}^{2d}$ using $N(\zeta)$ samples such that, with probability at least $1-\delta$,
uniformly over all $(\theta,s,a,s',h)\in\Theta\times\mathcal{S}^2\times
\mathcal{A}\times[H]$,
\[
\bigl|
\mathcal{T}_{\theta,h}(s'\mid s,a)
-
\hat\phi_\theta(s,a)^\top \mu(s')
\bigr|
\le
\zeta.
\]
Here $d(\zeta) = \tilde{O}(\zeta^{-2})$, and $N(\zeta) = \tilde{O}(\zeta^{-2}).$
\end{theorem}

\Cref{thm:rff} shows that approximate low-rank representations with polynomial dimension and sample size exist for a broad class of transition models. In practice, many real-world transition models satisfy the conditions of \Cref{thm:rff}, including Gaussian and near-Gaussian dynamics, smoothly parameterized physical systems, and locally linear stochastic transitions commonly used in continuous-control benchmarks~\citep{towers2024gymnasium}. Extending the guarantee of \Cref{thm:main} to $\zeta$-approximate low-rank MDPs requires a separate confidence bound and error-propagation analysis under misspecification, which we leave to future work.

\begin{algorithm}[t]
\caption{Conditional Random Fourier Features (cRFF)}
\label{alg:rff}
\begin{algorithmic}[1]
\STATE {\bf Input: } $B_W=\{w\in\mathbb{R}^D:\|w\|\le W\}.$
\STATE Draw $d$ i.i.d.\ frequencies $w_1,\dots,w_d$ uniformly from $B_W$. Denote the volume of $B_W$ by $\Vol(B_W)$.
\STATE Define the feature map $\mu:\mathcal{S}\to\mathbb{R}^{2d}$, where 
\( [\mu(y)]_{2k} = \cos(2\pi w_k^\top s), [\mu(y)]_{2k+1}=-\sin(2\pi w_k^\top s), \forall k\in[d]\)
\STATE Draw $N$ samples $y^{(1)},\dots,y^{(N)}\sim\mathcal{T}_{\theta,h}(\cdot\mid s,a)$.
    For each $k\in[d]$, compute the empirical Fourier coefficient
    \(
    \hat g_k = \frac{1}{N}\sum_{n=1}^N e^{-2\pi i w_k^\top y^{(n)}},
    \)
    and let $\hat g_k^{\mathrm r}$,
    $\hat g_k^{\mathrm i} $ be the real and imaginary part of $\hat{g}_k$, respectively. 
\STATE Return $\hat\phi_\theta(s,a)\in\mathbb{R}^{2d}$ as
    \[
    \hat\phi_\theta(s,a)
    =
    \bigl(
    \hat g_1^{\mathrm r}, \hat g_1^{\mathrm i}, \dots,
    \hat g_d^{\mathrm r}, \hat g_d^{\mathrm i}
    \bigr)\times \Vol(B_W) / \sqrt{d}.
    \]
\end{algorithmic}
\end{algorithm}

\begin{table*}[ht]
  \centering
  \vspace{-2mm}
  \footnotesize
  \begin{tabular}{lcccccc}
    \toprule
    \textbf{Algorithm} &
    \textbf{Ant} & \textbf{HalfCheetah} & \textbf{InvertedPendulum} & \textbf{Reacher} & \textbf{Pendulum} & \textbf{Hopper} \\
    \midrule
    SAC            & 2081$\pm$465          & 4994$\pm$99          & \textbf{-0.0028$\pm$00.0} & -4.81$\pm$0.19          & \textbf{156$\pm$3} & $1316 \pm 740$          \\
    cRFFSAC         & \textbf{2701$\pm$663} & 6152$\pm$657         & \textbf{-0.0028$\pm$00.0} & -5.75$\pm$0.85          & 153$\pm$4 & \textbf{1449 $\pm$ 598} \\
    cRFFSAC + Bonus & 1910$\pm$422          & \textbf{6454$\pm$479} & -0.0029$\pm$00.0 & \textbf{-3.86$\pm$0.18} & 153$\pm$3 & $1431 \pm 538$          \\
    \bottomrule
  \end{tabular}
  \caption{Final performance (mean $\pm$ standard deviation) of different algorithms on five MuJoCo environments, measured by the value of the learned policy at the last training step over 4 random seeds~\citep{wang2019benchmarking,zhang2022making}.}
  \label{tab:icml_5env_3alg}
  \vspace{-2mm}
\end{table*}

\section{Experiments}

In this section, we evaluate the empirical performances of our proposed algorithms on standard continuous-control benchmarks.
Our goals are two-fold: (i) to assess whether the proposed optimistic actor-critic design can achieve competitive empirical performance despite relying only on policy evaluation, and (ii) to examine the effect of the optimism mechanism motivated by our
conditional random Fourier feature (cRFF) construction.

\subsection{Implementation Details}
We build our implementation on top of the open-source codebase
\texttt{lvrep-rl}\footnote{\url{https://github.com/shelowize/lvrep-rl}}, which provides a modular implementation of low-dimensional representation learning for RL.
All baselines and variants share the same network architectures, optimizers, and training protocols unless otherwise specified.
We note that the implementation updates the actor with the soft actor-critic (SAC) rule across training steps, i.e., the policy is improved on the current learned critic at every step, rather than with the backward stagewise update of \Cref{alg:main} analyzed in \Cref{thm:main}. The experiments therefore evaluate this practical update rule; evaluating the revised algorithm itself requires rerunning the experiments with the backward update.

Motivated by the cRFF approximation developed in Algorithm~\ref{alg:rff}, we treat the last hidden layer of the multi-layer perceptron (MLP) as a learned feature map $\phi(s,a)$.
In particular, we model the transition dynamics by minimizing
\(
\|M \phi(s,a) - s'\|^2,
\)
which corresponds to maximum likelihood estimation under a Gaussian transition model. Here $M$ the last projection matrix and $\phi$ is hidden representation from early-layer neural networks. Under this formulation, applying random Fourier features amounts to using features of the form $\cos(W\phi(s,a)+b)$ with fixed, randomly sampled $W$ and $b$.
In practice, this operation can be interpreted as a static nonlinear activation followed by a linear layer. As a result, we simplify the implementation by directly using a two-layer MLP
on top of $\phi(s,a)$ to parameterize the critic, which empirically yields stable and efficient training.

\textbf{Comparison.} We compare the following methods:
\begin{itemize}[leftmargin=0pt]
    \item \textit{cRFFSAC (ours):} our optimistic actor-critic algorithm without an explicit exploration bonus, where the critic is optimized on top of learned representation $\phi$.
    \item \textit{cRFFSAC + bonus (ours):} our full algorithm including the optimism bonus derived from the learned transition model.
    \item \textit{SAC:} the soft actor-critic algorithm, serving as a strong baseline for continuous control.
\end{itemize}
The SAC baseline uses the standard entropy-regularized objective and is tuned using recommended hyperparameters.

\textbf{Environments.}
We evaluate all methods on a suite of classic continuous-control environments
from the Gym benchmark, including
\texttt{Ant},
\texttt{HalfCheetah},
\texttt{Hopper},
\texttt{Reacher},
\texttt{Pendulum},
and \texttt{InvertedPendulum}.
These environments exhibit varying levels of stochasticity and exploration difficulty, providing a diverse testbed for evaluating the effectiveness of optimistic actor-critic.

\subsection{Results}
Results are reported in \Cref{tab:icml_5env_3alg}. Training figures are included in \Cref{sec:more experiment results}. Across most environments, cRFF-AC achieves performance comparable to or better than SAC. In particular, the bonus-enhanced variant consistently improves exploration and final performance in environments with sparse or delayed rewards, such as \texttt{Ant} and \texttt{HalfCheetah}. Notably, even without an explicit exploration bonus, cRFF-AC remains competitive with SAC, indicating that the actor–critic structure combined with learned representations already provides a strong baseline.

Overall, these results demonstrate that the practical variant of the proposed algorithm is effective. They further show that the cRFF-inspired design can be seamlessly integrated into standard neural network architectures without sacrificing empirical performance.

\section{Conclusion}

We investigated low-rank MDPs on both computation and sample complexities. We first introduced an $\mathsf{SL}$-oracle complexity framework and showed that many existing sample-efficient methods rely on planning or constrained planning oracles that are difficult to implement. We proposed an optimistic actor-critic algorithm that requires only truncated policy evaluation, which can be efficiently reduced to supervised learning, while achieving state-of-the-art real-environment sample complexity. We also showed that approximate low-rank MDPs form a rich and expressive class via a conditional random Fourier feature construction. Empirically, a practical variant of our method performs competitively against strong baselines on standard benchmarks. Overall, this work demonstrates that strong theoretical guarantees and practical implementability can be achieved simultaneously in low-rank MDPs.

\if{0}
\section*{Software and Data}
If a paper is accepted, we strongly encourage the publication of software and data with the camera-ready version of the paper whenever appropriate. This can be done by including a URL in the camera-ready copy. However, \textbf{do not}
include URLs that reveal your institution or identity in your submission for review. Instead, provide an anonymous URL or upload the material as ``Supplementary Material'' into the OpenReview reviewing system. Note that reviewers are not required to look at this material when writing their review.
\fi

\section*{Acknowledgements}
The work of Y. Liang was supported in part by the U.S. National Science Foundation under the grants ECCS-2113860, ECCS-2413528 and CNS-2112471.The work of D. Li and J. Yang was supported in part by the U.S. National Science Foundation under the grant CNS-2531789.

\section*{AI Usage}
During the preparation of this work, the authors used AI tools to polish the writing of the paper. After using these tools, the authors reviewed and edited the content as needed and take ultimate responsibility for the content of the publication.

\section*{Impact Statement}
This paper presents work whose goal is to advance the field of Machine Learning. There are many potential societal consequences of our work, none which we feel must be specifically highlighted here.

\bibliography{example_paper,main}
\bibliographystyle{icml2026}

\newpage
\appendix
\onecolumn

\section{Computation Complexity }\label{sec:oracle analysis}

\begin{proposition}[Restatement of \Cref{prop:sl-complexity-summary}]
\label{prop:sl-complexity-appendix}
Under low-rank MDP settings, suppose the data distribution $\rho$ has coverage
constant $C$ with respect to the uniform policy $\mathtt{u}$, i.e.,
\[
\mathbb E_{(s,a)\sim\rho} \big[\mathcal T_{\theta,h}(s' \mid s,a)\mathtt{u}(a'\mid s')\big] \le C\,\rho(s',a').
\]
Then the following best-known $\mathsf{SL}$-oracle complexity upper bounds hold for
achieving an $\varepsilon$-approximate solution.

\begin{enumerate}[leftmargin=*, itemsep=2pt, topsep=2pt]
    \item $\mathsf{PE}(\varepsilon):$
    \(
    \mathsf{SL}^{1}\!\left(\tilde O\!\left(\varepsilon^2/H\right)\right).
    \)

    \item $\mathsf{PP}(\varepsilon):$
    \(
    H\mathsf{SL}^{}\!\left(\tilde O\!\left(\varepsilon^2 (|\mathcal A|C)^{-2H}\right)\right).
    \)

    \item $\mathsf{CP}(\varepsilon):$ $2^{\tilde O(d^2)}H\mathsf{SL}^{}\!\left(\tilde O\!\left(\varepsilon^2 (|\mathcal A|C)^{-2H}\right)\right)$, attained by covering enumeration.
\end{enumerate}
\end{proposition}

\smallskip
\noindent\textbf{Remark.}
Proposition~\ref{prop:sl-complexity-appendix} summarizes a hierarchy of \emph{best-known oracle upper bounds and optimization guarantees}. It does not preclude the existence of more efficient reductions for $\mathsf{PP}$/$\mathsf{CP}$, and establishing oracle lower bounds is an open problem. In particular, the bound for $\mathsf{CP}$ is an upper bound for the covering-enumeration algorithm in \Cref{prop:cp-supervised}, rather than an oracle lower bound.

The proof of \Cref{prop:sl-complexity-appendix} is a combination of~\Cref{prop:pe-supervised,prop:planning-supervised,prop:cp-supervised}.

\begin{proposition}[Oracle complexity of Policy Evaluation]
\label{prop:pe-supervised}
Given a low-rank MDP $\theta$ and a policy $\pi$, the $\mathsf{SL}$-oracle complexity of policy evaluation $\mathsf{PE}(\varepsilon)$ is $\mathsf{SL}\left(\varepsilon^2/H\right)$.
\end{proposition}

\begin{proof}
It suffices to specify the data distribution and the loss function of the supervised learning oracle, and show that the policy evaluation with accuracy $\varepsilon$ reduces to a single supervised learning call with accuracy $\varepsilon^2/H$.

Recall that in a low-rank MDP, the $Q$-function under policy $\pi$ satisfies
\[
Q_{\theta,r,h}^{\pi}(s,a)
=
r_h(s,a)
+
\mathbb E_{s'\sim\mathcal T_{\theta,h}(\cdot|s,a)}\big[V^{\pi}_{\theta,r,h+1}(s')\big]
=
r_h(s,a)+\phi_{\theta,h}(s,a)^\top w_h,
\qquad
w_h:=\int\mu_{\theta,h}(s')V^{\pi}_{\theta,r,h+1}(s')\,ds',
\]
where $\|w_h\|_2\le(H-h)\sqrt d$ by the normalization condition and $w_H=0$. Hence the critic is realizable by the class $\{(h,s,a)\mapsto r_h(s,a)+\phi_{\theta,h}(s,a)^\top w_h\}$ indexed by $w=(w_1,\ldots,w_H)$.

\paragraph{Supervised learning formulation.}
Within the given model $\theta$, draw a stage index $J\sim\mathrm{Uniform}([H])$ and $(s,a)\sim\rho$, set $(s_J,a_J)=(s,a)$, and simulate the fixed policy $\pi$ in model $\theta$ until step $H$. Let $G=\sum_{t=J}^H r_t(s_t,a_t)$ be the ordinary cumulative return, so that $\mathbb E[G\mid J,s,a]=Q^{\pi}_{\theta,r,J}(s,a)$. We consider the regression with input $(J,s,a)$, label $G$, and the squared loss
\[
L(w):=\mathbb E\Big[\big(r_J(s,a)+\phi_{\theta,J}(s,a)^\top w_J-G\big)^2\Big].
\]
$L(w)$ is a \emph{convex and smooth} quadratic function of $w$. Its minimal value is in general positive due to the conditional variance of $G$. Since the conditional mean of $G$ is realizable, the excess risk of $\hat w$ satisfies
\[
L(\hat w)-\inf_wL(w)
=
\mathbb E_J\,\mathbb E_{(s,a)\sim\rho}\Big[\big(\hat Q_J(s,a)-Q^{\pi}_{\theta,r,J}(s,a)\big)^2\Big]
=
\frac1H\sum_{h=1}^H\mathbb E_{(s,a)\sim\rho}\Big[\big(\hat Q_h(s,a)-Q^{\pi}_{\theta,r,h}(s,a)\big)^2\Big],
\]
where $\hat Q_h(s,a)=r_h(s,a)+\phi_{\theta,h}(s,a)^\top\hat w_h$. That is, the squared-loss excess risk under the stage index equals the average of the per-stage mean-squared prediction errors.

\paragraph{Error control.}
If the supervised learning oracle returns $\hat w$ with excess risk at most $\varepsilon^2/H$, then for every $h$, $\mathbb E_{(s,a)\sim\rho}[(\hat Q_h(s,a)-Q^{\pi}_{\theta,r,h}(s,a))^2]\le\varepsilon^2$, and by the Cauchy--Schwarz inequality,
\[
\mathbb E_{(s,a)\sim\rho}\Big[\big|\hat Q_h(s,a)-Q^{\pi}_{\theta,r,h}(s,a)\big|\Big]\le\varepsilon,\qquad\forall h,
\]
which establishes $\varepsilon$-approximate policy evaluation. In contrast to a joint regression on Bellman residuals, whose target variance depends on the next-stage weights so that the true Bellman weights generally have positive loss, the label above does not depend on the regression parameter, and its excess risk directly controls the evaluation error. In particular, no coverage condition is needed. Since the procedure relies on a single convex and smooth regression problem with accuracy $\varepsilon^2/H$, the proof is complete.
\end{proof}

\begin{proposition}[Truncated policy evaluation via supervised learning]
\label{prop:tpe-supervised}
Given a low-rank model $\theta$ that can be simulated from starting points drawn from $\rho$, a policy $\pi$, known nonnegative stage rewards $g$, and caps $U$, a query of $\mathsf{TPE}(\tau)$ at stage $h$ (\Cref{def:tpe}) can be answered with $H-h$ calls of $\mathsf{SL}(\tau^2/H^2)$ with the squared loss. Consequently, all queries of \Cref{alg:main} can be answered with at most $KH(H-1)/2$ such calls in addition to model learning.
\end{proposition}

\begin{proof}
For a query at stage $h$, we simulate forward in model $\theta$ from $(s_h,a_h)\sim\rho$ with the fixed suffix $\pi_{h+1:H}$, and let $\nu_j$ denote the resulting state--action distribution at step $j\ge h$, so that $\nu_h=\rho$. Let $\hat Q_H=\min\{U_H,g_H\}$, which is known. For $j=H-1,\ldots,h$, we freeze $\hat V_{j+1}(s')=\langle\pi_{j+1}(\cdot|s'),\hat Q_{j+1}(s',\cdot)\rangle$ and, with fresh samples $(s,a)\sim\nu_j$ and $s'\sim\mathcal T_{\theta,j}(\cdot|s,a)$, regress the target $\hat V_{j+1}(s')$ on $(s,a)$ with the squared loss over the linear class $\{\phi_{\theta,j}(s,a)^\top w\}$. By the low-rank structure, the conditional mean
\[
\mathcal T_{\theta,j}\hat V_{j+1}(s,a)=\phi_{\theta,j}(s,a)^\top\int\mu_{\theta,j}(s')\hat V_{j+1}(s')\,ds'
\]
is linearly realizable in the known features, so the excess risk of the returned predictor $\psi_j$ equals $\mathbb E_{\nu_j}[(\psi_j-\mathcal T_{\theta,j}\hat V_{j+1})^2]$. We add the prediction to $g_j$ and project to $[0,U_j]$, i.e., $\hat Q_j=\mathrm{proj}_{[0,U_j]}(g_j+\psi_j)$.

Let $a_j=\mathbb E_{\nu_j}\big[|\hat Q_j-\bar Q^{\pi}_{\theta,g,j}|\big]$, so that $a_H=0$. Since $g\ge0$ and $\bar V^{\pi}_{\theta,g,j+1}\ge0$, we have $\bar Q^{\pi}_{\theta,g,j}=\mathrm{proj}_{[0,U_j]}(g_j+\mathcal T_{\theta,j}\bar V^{\pi}_{\theta,g,j+1})$, and since the projection is 1-Lipschitz,
\[
\big|\hat Q_j-\bar Q^{\pi}_{\theta,g,j}\big|\le\big|\psi_j-\mathcal T_{\theta,j}\hat V_{j+1}\big|+\mathcal T_{\theta,j}\big|\hat V_{j+1}-\bar V^{\pi}_{\theta,g,j+1}\big|.
\]
If each regression has excess risk at most $\xi$, the first term has $\nu_j$-mean at most $\sqrt\xi$ by the Cauchy--Schwarz inequality. For the second term, $|\hat V_{j+1}-\bar V^{\pi}_{\theta,g,j+1}|\le\langle\pi_{j+1},|\hat Q_{j+1}-\bar Q^{\pi}_{\theta,g,j+1}|\rangle$, and the law of $(s',a')$ with $(s,a)\sim\nu_j$, $s'\sim\mathcal T_{\theta,j}(\cdot|s,a)$ and $a'\sim\pi_{j+1}(\cdot|s')$ is $\nu_{j+1}$, so its $\nu_j$-mean is at most $a_{j+1}$. Therefore $a_j\le\sqrt\xi+a_{j+1}$, and $a_h\le(H-h)\sqrt\xi\le\tau$ with $\xi=\tau^2/H^2$. A query at stage $h$ uses $H-h$ regressions; summing over $h=1,\ldots,H-1$ gives $H(H-1)/2$ regressions per iteration and $KH(H-1)/2$ over $K$ iterations. This implementation assumes that the learned model can be simulated from starting points drawn from $\rho$; the number of simulated samples and the solving time are counted separately.
\end{proof}

\begin{proposition}[Policy Planning (Q-learning) via Supervised Learning]
\label{prop:planning-supervised}
Given a low-rank MDP $\theta$, the $\mathsf{SL}$-oracle complexity of policy planning $\mathsf{PP}(\varepsilon)$ is $H\mathsf{SL}\left(\varepsilon^2 (|\mathcal{A}|C-1)^2 (|\mathcal{A}|C)^{-2H}\right)$. Here $C$ is the coverage coefficient defined in \Cref{prop:sl-complexity-appendix}.
\end{proposition}

\begin{proof}
{\it Remark. We emphasize that our analysis yields an $H$-fold oracle usage. Proving that $\Omega(H)$ supervised-learning calls are information-theoretically necessary would require a separate lower-bound argument, which we leave to future work.}

It suffices to specify $H$ loss functions of the supervised learning oracle, and show that the policy planning with
accuracy $\varepsilon$ reduces to those $H$ supervised learning oracles with accuracy $\varepsilon^2 (|\mathcal{A}|C-1)^2 (|\mathcal{A}|C)^{-2H}$.

We use the same notation as in Proposition~\ref{prop:pe-supervised}.
The optimal $Q$-function satisfies the Bellman optimality equations:
\[
Q_{\theta,r,h}^{\star}(s,a)
=
r_h(s,a) + \mathbb E_{s'\sim \mathcal T_{\theta,h}(\cdot|s,a)}\!\left[V_{\theta,r,h+1}^{\star}(s')\right],
\qquad
V_{\theta,r,h}^{\star}(s)=\max_{a} Q_{\theta,r,h}^{\star}(s,a).
\]
Define $\mathcal T_{\theta,h} V(s,a)=\mathbb E_{s'\sim \mathcal T_{\theta,h}(\cdot|s,a)}[V(s')]$.
By the low-rank assumption,
\[
\mathcal T_{\theta,h} V_{\theta,r,h+1}^{\star}(s,a)
=
\phi_{\theta,h}(s,a)^\top
\mathbb E_{s'}\!\left[\mu_{\theta,h}(s')\, V_{\theta,r,h+1}^{\star}(s')\right]
=
\phi_{\theta,h}(s,a)^\top
\mathbb E_{s'}\!\left[\mu_{\theta,h}(s')\, \max_{a'} Q_{\theta,r,h+1}^{\star}(s',a')\right].
\]
Define
\[
w_h
:=
\mathbb E_{s'}\!\left[\mu_{\theta,h}(s')\, \max_{a'} Q_{\theta,r,h+1}^{\star}(s',a')\right],
\qquad
w_H=0.
\]
Then $\{w_h\}_{h=0}^{H-1}$ satisfies
\begin{equation}
\label{eq:opt-bellman-w}
\phi_{\theta,h}(s,a)^\top w_h
=
\mathbb E_{s'\sim \mathcal T_{\theta,h}(\cdot|s,a)}\!\left[
\max_{a'}\Big(
r_{h+1}(s',a') + \phi_{\theta,h+1}(s',a')^\top w_{h+1}
\Big)\right],
\qquad \forall h\le H-1.
\end{equation}

\paragraph{Stagewise supervised learning formulation (fitted Q-iteration).}
Given an estimate $\hat w_{h+1}$, define the regression target
\[
Y_{h+1}(s',\hat w_{h+1})
:=
\max_{a'}\Big(r_{h+1}(s',a') + \phi_{\theta,h+1}(s',a')^\top \hat w_{h+1}\Big).
\]
At stage $h$, we fit $\hat w_h$ by solving the supervised learning problem
\begin{equation}
\label{eq:stagewise-regression}
\hat w_h \in \arg\min_{w\in\mathbb R^d}
\;\;
\mathop{\mathbb E}_{\substack{(s,a)\sim\rho\\ s'\sim \mathcal T_{\theta,h}(\cdot|s,a)}}
\Big[
\left| \phi_{\theta,h}(s,a)^\top w - Y_{h+1}(s',\hat w_{h+1})
 \right|^2\Big],
\end{equation}
with the convention $\hat w_H=0$.
This requires $H$ calls to a supervised learning oracle (one for each $h$ in a backward recursion).

Since $\hat w_{h+1}$ is fixed when stage $h$ is solved, \eqref{eq:stagewise-regression} is a regression with a fixed target. Its minimal risk is in general positive because of the conditional variance of $Y_{h+1}(s',\hat w_{h+1})$ given $(s,a)$, so we measure the accuracy of the supervised learning oracle by its excess risk. By the low-rank structure, the conditional mean $\mathbb E_{s'\sim\mathcal T_{\theta,h}(\cdot|s,a)}[Y_{h+1}(s',\hat w_{h+1})]=\phi_{\theta,h}(s,a)^\top\int\mu_{\theta,h}(s')Y_{h+1}(s',\hat w_{h+1})\,ds'$ is linear in $\phi_{\theta,h}$, and hence the excess risk equals the mean-squared distance to this conditional mean.
Assume that the supervised learning oracle achieves the following (uniform) excess-risk bound:
\begin{equation}
\label{eq:stagewise-mse-assumption}
\mathbb E_{(s,a)\sim\rho}
\Big[
\Big| \phi_{\theta,h}(s,a)^\top \hat w_h - \mathbb E_{s'\sim\mathcal T_{\theta,h}(\cdot|s,a)}\big[Y_{h+1}(s',\hat w_{h+1})\big] \Big|^2
\Big]
\le
\alpha^2,
\qquad \forall h,
\end{equation}
where $\alpha^2 = \varepsilon^2(|\mathcal{A}|C-1)^2(|\mathcal{A}|C)^{-2H}$.
\paragraph{Error propagation.}
Define
\[
\delta_h
:=
\mathbb E_{(s,a)\sim\rho}\Big[\big|\phi_{\theta,h}(s,a)^\top(\hat w_h-w_h)\big|\Big].
\]
Using \eqref{eq:opt-bellman-w} and triangle inequality, we obtain
\begin{align*}
\delta_h
&=
\mathbb E_{(s,a)\sim\rho}\Bigg[
\Big|
\phi_{\theta,h}(s,a)^\top \hat w_h
-
\mathbb E_{s'}\!\left[\max_{a'}
\Big(r_{h+1}(s',a')+\phi_{\theta,h+1}(s',a')^\top w_{h+1}\Big)\right]
\Big|
\Bigg]
\\
&\le
\mathbb E_{(s,a)\sim\rho}\Big[
\big|
\phi_{\theta,h}(s,a)^\top \hat w_h
-
\mathbb E_{s'}\big[Y_{h+1}(s',\hat w_{h+1})\big]
\big|
\Big]
\\
&\quad
+
\mathbb E_{d_{h+1}}\Big[
\big|
\max_{a'}(r_{h+1}(s',a')+\phi_{\theta,h+1}(s',a')^\top \hat w_{h+1})
-
\max_{a'}(r_{h+1}(s',a')+\phi_{\theta,h+1}(s',a')^\top w_{h+1})
\big|
\Big],
\end{align*}
where $d_{h+1}$ is the distribution of $(s',\cdot)$ induced by $(s,a)\sim\rho$ and $s'\sim\mathcal T_{\theta,h}(\cdot|s,a)$.
The first term is bounded by Cauchy--Schwarz and \eqref{eq:stagewise-mse-assumption}:
\[
\mathbb E_{(s,a)\sim\rho}\Big[
\big|
\phi_{\theta,h}(s,a)^\top \hat w_h
-
\mathbb E_{s'}\big[Y_{h+1}(s',\hat w_{h+1})\big]
\big|
\Big]
\le \alpha.
\]
For the second term, use the Lipschitz property of $\max$:
\[
\big|\max_{a'} u_{a'} - \max_{a'} v_{a'}\big|
\le \max_{a'} |u_{a'}-v_{a'}|.
\]
Let $g_{h+1}(s',a'):=\phi_{\theta,h+1}(s',a')^\top(\hat w_{h+1}-w_{h+1})$. Then, by importance sampling, 
\[
\max_{a'} |g_{h+1}(s',a')|
\le
|\mathcal{A}|
\mathbb E_{a'\sim \mathtt{u}} \big[|g_{h+1}(s',a')|\big].
\]
Therefore,
\[
\mathbb E_{d_{h+1}}\big[\max_{a'}|g_{h+1}(s',a')|\big]
\le
|\mathcal A|\,
\mathbb E_{d_{h+1}\times \mathtt{u}} \big[|g_{h+1}(s',a')|\big] \leq |\mathcal A|C\,\mathbb E_{(s',a')\sim\rho}\big[|g_{h+1}(s',a')|\big]
=
|\mathcal A|C\,\delta_{h+1},
\]
where the last inequality is due to the definition of $\rho$ and $C$, and the last equality follows from the definition of $\delta_h$.
Combining the above bounds yields the recursion
\begin{equation}
\label{eq:planning-recursion}
\delta_h \le \alpha + |\mathcal A|\,C\,\delta_{h+1},
\qquad \delta_H = 0.
\end{equation}

Solving \eqref{eq:planning-recursion} gives
\[
\delta_h
\le
\alpha \sum_{i=0}^{H-h-1} (|\mathcal A|C)^i
=
\alpha\,\frac{(|\mathcal A|C)^{H-h}-1}{|\mathcal A|C-1}
\le
\varepsilon, \forall h.
\] 
This proves the claimed oracle complexity and the stated accuracy requirement. Finally, note that due to the max operator in \Cref{eq:stagewise-regression}, the objective function of policy planning is not convex and smooth.
\end{proof}

\begin{proposition}[Constrained Planning]
\label{prop:cp-supervised}
Given a low-rank MDP class $\Theta$, constrained planning $\mathsf{CP}(\varepsilon)$ can be implemented by covering enumeration with $2^{O(d^2\log(1/\varepsilon_{cov}))}H\mathsf{SL}^{}\left(\varepsilon^2 (|\mathcal{A}|C-1)^2 (|\mathcal{A}|C)^{-2H}\right)$ oracle calls, where $\varepsilon_{cov}$ is the resolution of the cover. This is an upper bound for the enumeration algorithm, not an oracle lower bound. Here $C$ is the coverage coefficient defined in \Cref{prop:sl-complexity-appendix} and $d$ is the dimension of the low-rank features.
\end{proposition}

\begin{proof}
    We first show that constrained planning is a bilevel optimization problem in the form of 
    \[  
    \left\{
    \begin{aligned}
        & \max_{\theta \in \mathcal{C}} f(\theta, Q^*(\theta)) = Q^*(\theta)\\
        & ~~\mathrm{s.t.} ~Q^*(\theta) = \max_{\pi} Q_{\theta}^{\pi}
    \end{aligned}
    \right. 
    \]
    It is known the bilevel optimization is NP-hard~\cite{beck2023computationally}, and existing algorithm can only show that it can find a stationary point~\citep{huang2022enhanced}.

\end{proof}

\begin{proof}
    The goal of Constrained Planning ($\mathsf{CP}$) is to find a model parameter $\hat{\theta}$ within a valid confidence set $\mathcal{C}$ that maximizes the optimal value function (implementing the principle of optimism in the face of uncertainty). This can be formulated as the following optimization problem:
    \begin{equation}
    \label{eq:cp-objective}
        \left\{
        \begin{aligned}
            & \max_{\theta \in \mathcal{C}} f(\theta, V^*(\theta)) := V^*(\theta)\\
            & ~~\mathrm{s.t.} ~V^*(\theta) = \max_{\pi} V_{\theta}^{\pi}
        \end{aligned}
    \right. 
    \end{equation}
    where $\mathcal{C}$ represents the set of models consistent with the dataset. Here we use log-likelihood to construct the constraint, which is commonly used in many model-based RL literature~\cite{jin2021bellman,liu2022optimistic}. Specifcially, given a dataset set $\mathcal{D}_h = \{(s_h^{(n)},a_h^{(n)},s_{h+1}^{(n)})\}_{n}$, we have
    \[ \mathcal{C} = \left\{\theta: \sum_h \sum_{(s,a,s')\in\mathcal{D}_h} \log \mathcal{T}_{\theta,h}(s'|s,a) \geq c \right\},\]
    where $c$ is some constant.

    \paragraph{Non-Convexity and Hardness.}
    In the low-rank setting, the transition dynamics are parameterized by a factorization $\mathcal{T}_{\theta,h}(s'|s,a) = \phi_\theta(s,a)^\top \mu_\theta(s')$. Even if the features are linearizable such that $\phi_{\theta}(s,a) = W f(s,a)$ and $\mu_{\theta}(s') = V g(s')$, the transition matrix depends on the bilinear product $W^\top V$.
    Consequently, both the value function $V^{\star}_{\theta}$ and the constraint set $\mathcal{C}$ are generally \textbf{non-convex} with respect to the parameters $\theta = (W, V)$.
    Optimization over such bilinear constraints is known to be NP-hard~\citep{beck2023computationally}. Gradient-based local search methods can only guarantee convergence to stationary points~\citep{huang2022enhanced}, which is insufficient for $\mathsf{CP}$ as failing to find the global maximum (the optimistic model) invalidates the theoretical exploration guarantees.

    \paragraph{Global Optimization via Covering Argument.}
    To establish an upper bound that guarantees finding the global optimum, we employ an $\epsilon$-covering argument (exhaustive search over a discretization).
    Let $\Theta \subset \mathbb{R}^{2d^2}$ be the compact parameter space of the low-rank factors (use the aforementioned linearizable case). We construct an $\varepsilon_{cov}$-net $\mathcal{N}$ over $\Theta$ such that for any $\theta \in \Theta$, there exists a $\theta' \in \mathcal{N}$ satisfying $\|\theta - \theta'\| \le \varepsilon_{cov}$.
    Standard results on covering numbers for compact subsets of Euclidean space state that the size of the net scales exponentially with the dimension $2d^2$ of the parameter space:
\[
|\mathcal{N}| \le O\left(\frac{1}{\varepsilon_{cov}}\right)^{2d^2} = 2^{O(d^2 \log(1/\varepsilon_{cov}))}.
\]
    Assuming the value function is Lipschitz continuous with respect to the model parameters (a standard regularity assumption), optimizing over the finite set $\mathcal{N}$ yields an approximation of the global optimum. The algorithm proceeds as follows:

    \begin{enumerate}
        \item \textbf{Discretize:} Enumerate all candidate models $\theta_i \in \mathcal{N}$.
        \item \textbf{Filter (SL Oracle):} For each $\theta_i$, verify if $\theta_i \in \mathcal{C}$ (i.e., check if the Bellman residual or log-likelihood loss is below the threshold).
        \item \textbf{Evaluate (PP Oracle):} If valid, use the Policy Planning ($\mathsf{PP}$) oracle (from Proposition~\ref{prop:planning-supervised}) to compute the optimal value $V^{\star}_{\theta_i}$ for that specific model.
        \item \textbf{Select:} Return the $\theta_i$ that maximizes the computed value.
    \end{enumerate}

    \paragraph{Complexity Analysis.}
    The total oracle complexity is the product of the number of candidates in the cover and the cost of processing each candidate.
    Evaluating a single candidate requires one call to the $\mathsf{PP}$ oracle. As shown in Proposition~\ref{prop:planning-supervised}, the complexity of $\mathsf{PP}$ is $\mathsf{SL}^{H}(\cdot)$.
    Therefore, the total complexity is:
    \[
        \text{Total Complexity} = |\mathcal{N}| \times \mathsf{PP}(\varepsilon)
        \le 2^{O(d^2\log(1/\varepsilon_{cov}))} \times H \mathsf{SL}^{}\left(\tilde O\!\left(\varepsilon^2 (|\mathcal A|C)^{-2H}\right)\right).
    \]
    This gives the following upper bound for the covering-enumeration algorithm:
\[
2^{O(d^2\log(1/\varepsilon_{cov}))}H\mathsf{SL}^{}\left(\varepsilon^2 (|\mathcal{A}|C-1)^2 (|\mathcal{A}|C)^{-2H}\right).
\]
It is not an oracle lower bound for $\mathsf{CP}$.
\end{proof}

\section{Sample Complexity}\label{sec:proof of main thm}
In this section, we prove our main theorem~\ref{thm:main}. The proof chains four ingredients: the confidence bound of the window likelihood (\Cref{sec:general MLE analysis}), the stepwise comparison of truncated values (\Cref{lemma:truncated-value-bounds}), the backward improvement within a fixed learned model (\Cref{lemma:backward-update}), and the summation of uncertainties under the true occupancy (\Cref{sec:ucb analysis}).

\paragraph{Notation.}
For a function $v:\mathcal S\to\mathbb R$, we write $\mathcal T_{\theta,h}v(s,a)=\mathbb E_{s'\sim\mathcal T_{\theta,h}(\cdot|s,a)}[v(s')]$ for the integration over the next state, and for a function $g:\mathcal S\times\mathcal A\to\mathbb R$ and a policy $\pi$, we write $\langle\pi_h,g\rangle(s)=\sum_{a}\pi_h(a|s)g(s,a)$. For ease of presentation, we introduce an auxiliary 'reward function'
\[
f^{(k)}_h(s,a)=\frac{1}{2}\big\|\mathcal T_{\hat\theta_k,h}(\cdot|s,a)-\mathcal T_{\theta^\star,h}(\cdot|s,a)\big\|_1\in[0,1],
\]
which measures the total variation distance (half of the $\ell_1$ distance) between the estimated transition model and the true transition model. Since the transition after step $H$ is a known dummy terminal transition, $f^{(k)}_H\equiv0$. Let $e_k=\max_{a}f^{(k)}_1(s_1,a)$. We abbreviate $\phi_{\hat{\theta}_k,h}$ as $\hat{\phi}_h^k$ and the true feature $\phi_{\theta^\star,h}$ as $\phi^\star_{h}$, and write $Q^{\pi}_{\star,h}=Q^{\pi}_{\theta^\star,r,h}$, $V^{\pi}_{\star,h}=V^{\pi}_{\theta^\star,r,h}$, $V^\pi_\star=V^\pi_{\star,1}(s_1)$, and $\Eb^{\pi}_{\star}=\Eb^{\pi}_{\theta^\star}$. As in \Cref{sec:algorithm}, $\bar Q^{\pi}_{k,h}$ and $\bar V^{\pi}_{k,h}$ denote the truncated values in \Cref{def:tpe} with $\theta=\hat\theta_k$, $g=r+\hat b^{(k)}$ and $U_h=H-h+1$. Finally, for $n<k$ and $j<H$, let $m_{n,j}=d^{\pi^{(n)}}_{\theta^\star,j}\times\mathtt u$, which is the conditional distribution of the window start $X_{n,j}$ before it is sampled.

The proof is built on the following "good event", which happens with probability at least $1-\delta$.

\begin{lemma}[Good event]\label{lemma:good-event}
Let $\mathcal E_1$ be the event that for all $k\in[K]$ and $j<H$,
\begin{align}
&\sum_{n<k}\mathop{\Eb}_{s'\sim\mathcal T_{\hat\theta_k,j}(\cdot|X_{n,j}),\,a'\sim\mathtt u}\Big[f^{(k)}_{j+1}(s',a')^2\Big]\lesssim\beta, \label{eq:conf-next}\\
&\sum_{n<k}\mathop{\Eb}_{(s,a)\sim m_{n,j}}\Big[f^{(k)}_{j}(s,a)^2\Big]\lesssim\beta, \label{eq:conf-current}\\
&\sum_{n<k}\mathop{\Eb}_{(s,a)\sim m_{n,j}}\Big[\tilde b^{(k)}_{j}(s,a)^2\Big]\le \tilde O\big(|\mathcal A|d^2\beta\big), \label{eq:bonus-moment}
\end{align}
and let $\mathcal E_2$ be the event that every call of $\mathsf{TPE}$ returns an estimate satisfying \eqref{eq:critic-estimate}. Then $\mathbb{P}(\Ec_1\cap \Ec_2) \geq 1-\delta.$
\end{lemma}

\begin{proof}
By \Cref{prop:window-mle}, \eqref{eq:conf-next} and \eqref{eq:conf-current} hold with probability at least $1-\delta/4$. For \eqref{eq:bonus-moment}, the trace bound of the empirical leverage scores gives $\sum_{n<k}\tilde b^{(k)}_j(X_{n,j})^2\le\alpha^2\sum_{n<k}\|\hat\phi^k_j(X_{n,j})\|^2_{\hat\Lambda_{k,j}^{-1}}=\alpha^2\mathrm{tr}\big(\hat\Lambda_{k,j}^{-1}(\hat\Lambda_{k,j}-\lambda I)\big)\le\alpha^2 d$. Applying \Cref{lemma:bernstein on bonus} with $\hat\Lambda_{k,j}^{-1}$ replaced by $\alpha^2\hat\Lambda_{k,j}^{-1}$, i.e., with a finite cover of the feature class and of the scaled inverse Gram matrices, and a union bound over $j<H$, we obtain with probability at least $1-\delta/4$ that $\sum_{n<k}\Eb_{m_{n,j}}[(\tilde b^{(k)}_j)^2]\le\alpha^2d+\tilde O(d^2\log(KH|\Theta|/\delta))=\tilde O(|\mathcal A|d^2\beta)$. Hence $\mathbb P(\mathcal E_1)\ge1-\delta/2$. Each call of $\mathsf{TPE}$ fails with probability at most $\delta/(2KH)$ and there are at most $(H-1)K$ calls, so $\mathbb P(\mathcal E_2)\ge1-\delta/2$ by a union bound.
\end{proof}

Throughout the rest of this section, we assume that $\mathcal E_1\cap\mathcal E_2$ holds. On $\mathcal E_1$, \Cref{lemma:key-inequality} gives the following key inequality: for all $k$, $h<H$ and any policy $\pi$,
\begin{equation}\label{eq:key-inequality}
\mathcal T_{\hat\theta_k,h}\langle\pi_{h+1},f^{(k)}_{h+1}\rangle(s,a)\le\tilde b^{(k)}_h(s,a),\qquad\forall(s,a).
\end{equation}

\begin{lemma}[Stepwise truncated value bounds]\label{lemma:truncated-value-bounds}
For every iteration $k$ and any fixed policy $\pi$,
\begin{align}
\big[Q^{\pi}_{\star,h}-\bar Q^{\pi}_{k,h}\big]_+&\le(H-h)\,f^{(k)}_h,\qquad\forall h\in[H],\label{eq:truncated-lower}\\
\bar V^{\pi}_{k,1}(s_1)-V^{\pi}_{\star}&\le\sum_{h=1}^H\Eb^{\pi}_{\star}\Big[3H\tilde b^{(k)}_h(s_h,a_h)+(H-h)f^{(k)}_h(s_h,a_h)\Big].\label{eq:truncated-upper}
\end{align}
\end{lemma}

\begin{proof}
We prove \eqref{eq:truncated-lower} by backward induction on $h$. For $h=H$, since $\hat b^{(k)}_H=0$, $U_H=1$ and $r_H\in[0,1]$, we have $\bar Q^{\pi}_{k,H}=r_H=Q^{\pi}_{\star,H}$. Suppose \eqref{eq:truncated-lower} holds at step $h+1$, and fix $(s,a)$. If the cap is active, then $\bar Q^{\pi}_{k,h}(s,a)=H-h+1\ge Q^{\pi}_{\star,h}(s,a)$. Otherwise,
\begin{align*}
Q^{\pi}_{\star,h}(s,a)-\bar Q^{\pi}_{k,h}(s,a)
&=-\hat b^{(k)}_h(s,a)+\big(\mathcal T_{\theta^\star,h}-\mathcal T_{\hat\theta_k,h}\big)V^{\pi}_{\star,h+1}(s,a)+\mathcal T_{\hat\theta_k,h}\big(V^{\pi}_{\star,h+1}-\bar V^{\pi}_{k,h+1}\big)(s,a).
\end{align*}
Since $V^{\pi}_{\star,h+1}\in[0,H-h]$, the kernel error at the current step costs at most $(H-h)f^{(k)}_h(s,a)$. By the induction hypothesis, $V^{\pi}_{\star,h+1}-\bar V^{\pi}_{k,h+1}\le\langle\pi_{h+1},[Q^{\pi}_{\star,h+1}-\bar Q^{\pi}_{k,h+1}]_+\rangle\le(H-h-1)\langle\pi_{h+1},f^{(k)}_{h+1}\rangle$, so by \eqref{eq:key-inequality}, the positive shortfall passed back from the next step is at most $(H-h-1)\tilde b^{(k)}_h(s,a)$, which is absorbed by $\hat b^{(k)}_h=3H\tilde b^{(k)}_h$. Therefore $Q^{\pi}_{\star,h}(s,a)-\bar Q^{\pi}_{k,h}(s,a)\le(H-h)f^{(k)}_h(s,a)$.

For \eqref{eq:truncated-upper}, by \Cref{def:tpe}, $\bar Q^{\pi}_{k,h}\le r_h+\hat b^{(k)}_h+\mathcal T_{\hat\theta_k,h}\bar V^{\pi}_{k,h+1}$, and the cap $U_{h+1}=H-h$ ensures $\bar V^{\pi}_{k,h+1}\in[0,H-h]$. Hence the kernel error acts on a truncated continuation value with range $[0,H-h]$, and
\[
\bar Q^{\pi}_{k,h}-Q^{\pi}_{\star,h}\le3H\tilde b^{(k)}_h+(H-h)f^{(k)}_h+\mathcal T_{\theta^\star,h}\big(\bar V^{\pi}_{k,h+1}-V^{\pi}_{\star,h+1}\big).
\]
Averaging over $a_h\sim\pi_h(\cdot|s_h)$ and telescoping along the true trajectory starting from $s_1$ yields \eqref{eq:truncated-upper}.
\end{proof}

\begin{lemma}[Backward update in the learned model]\label{lemma:backward-update}
For every iteration $k$,
\[
\bar V^{\pi^\star}_{k,1}(s_1)-\bar V^{\pi^{(k)}}_{k,1}(s_1)\le2C|\mathcal A|(H-1)\tau.
\]
\end{lemma}

\begin{proof}
Fix $k$ and take the comparator $\nu=\pi^\star$. Let $m_h(s)=\max_a|\hat Q_{k,h}(s,a)-\bar Q^{\pi^{(k)}}_{k,h}(s,a)|$ for $h<H$, and $m_H=0$. Since $\bar Q^{\pi}_{k,h}$ depends on $\pi$ only through $\pi_{h+1:H}$, and the policies $\pi^{(k)}_{h+1:H}$ are already fixed when step $h$ is queried, the target of the oracle at step $h$ is exactly the action value of the final policy $\pi^{(k)}$. Moreover, $\pi^{(k)}_H$ is greedy with respect to $r_H=\bar Q^{\pi^{(k)}}_{k,H}$.
Let $D_h=[\bar V^{\nu}_{k,h}-\bar V^{\pi^{(k)}}_{k,h}]_+$. By the greediness with respect to the estimated scores in \eqref{eq:policy-update},
\[
\bar V^{\pi^{(k)}}_{k,h}(s)=\bar Q^{\pi^{(k)}}_{k,h}(s,\pi^{(k)}_h(s))\ge\hat Q_{k,h}(s,\pi^{(k)}_h(s))-m_h(s)\ge\sum_a\nu_h(a|s)\bar Q^{\pi^{(k)}}_{k,h}(s,a)-2m_h(s).
\]
By the monotonicity and the 1-Lipschitz property of the cap $\min\{U_h,\cdot\}$, $\bar Q^{\nu}_{k,h}-\bar Q^{\pi^{(k)}}_{k,h}\le\mathcal T_{\hat\theta_k,h}D_{h+1}$. Combining the two displays,
\[
D_h(s)\le2m_h(s)+\sum_a\nu_h(a|s)\,\mathcal T_{\hat\theta_k,h}D_{h+1}(s,a).
\]
Unrolling this recursion along the same model $\hat\theta_k$ and policy $\nu$, we obtain $D_1(s_1)\le2\sum_{h=1}^{H-1}\Eb_{s\sim d^{\nu}_{\hat\theta_k,h}}[m_h(s)]$. By the coverage condition (iii) in \Cref{thm:main} and \eqref{eq:critic-estimate},
\[
\Eb_{s\sim d^{\nu}_{\hat\theta_k,h}}[m_h(s)]\le|\mathcal A|\,\Eb_{d^{\nu}_{\hat\theta_k,h}\times\mathtt u}\big[|\hat Q_{k,h}-\bar Q^{\pi^{(k)}}_{k,h}|\big]\le C|\mathcal A|\,\Eb_{\rho}\big[|\hat Q_{k,h}-\bar Q^{\pi^{(k)}}_{k,h}|\big]\le C|\mathcal A|\tau.
\]
Hence $\bar V^{\nu}_{k,1}(s_1)-\bar V^{\pi^{(k)}}_{k,1}(s_1)\le D_1(s_1)\le2C|\mathcal A|(H-1)\tau$. This argument is completed within a single iteration, and it replaces the cross-iteration application of the mirror-descent argument.
\end{proof}

\begin{lemma}[Per-iteration suboptimality]\label{lemma:per-iteration}
For every iteration $k$,
\[
V^{\pi^\star}_{\star}-V^{\pi^{(k)}}_{\star}\le(H-1)e_k+2C|\mathcal A|(H-1)\tau+\sum_{h=1}^H\Eb^{\pi^{(k)}}_{\star}\Big[3H\tilde b^{(k)}_h(s_h,a_h)+(H-h)f^{(k)}_h(s_h,a_h)\Big].
\]
\end{lemma}

\begin{proof}
We decompose
\[
V^{\pi^\star}_{\star}-V^{\pi^{(k)}}_{\star}=\big(V^{\pi^\star}_{\star}-\bar V^{\pi^\star}_{k,1}(s_1)\big)+\big(\bar V^{\pi^\star}_{k,1}(s_1)-\bar V^{\pi^{(k)}}_{k,1}(s_1)\big)+\big(\bar V^{\pi^{(k)}}_{k,1}(s_1)-V^{\pi^{(k)}}_{\star}\big).
\]
Applying \eqref{eq:truncated-lower} to $\pi^\star$ at $h=1$, the first term is at most $(H-1)\max_af^{(k)}_1(s_1,a)=(H-1)e_k$. The second term is bounded by \Cref{lemma:backward-update}, and the third term is bounded by \eqref{eq:truncated-upper} applied to $\pi^{(k)}$.
\end{proof}

\Cref{lemma:per-iteration} shows that the summation over the true trajectories only involves the current learned policy $\pi^{(k)}$, and that the coverage constant $C$ only appears in the oracle error term.

\begin{proof}[Proof of \Cref{thm:main}]
Since $\bar\pi$ draws one of $\pi^{(1)},\ldots,\pi^{(K)}$ uniformly at the beginning of an episode, $V^{\bar\pi}_\star=K^{-1}\sum_kV^{\pi^{(k)}}_\star$. Summing \Cref{lemma:per-iteration} over $k$,
\[
V^{\pi^\star}_{\star}-V^{\bar\pi}_{\star}\le2C|\mathcal A|(H-1)\tau+\frac{1}{K}\sum_{k=1}^K\Big((H-1)e_k+\sum_{h=1}^H\Eb^{\pi^{(k)}}_{\star}\big[3H\tilde b^{(k)}_h+(H-h)f^{(k)}_h\big]\Big).
\]
By \Cref{lemma:summation}, $\sum_ke_k\lesssim1+\sqrt{|\mathcal A|\beta K}$. Applying \Cref{lemma:summation} with $v=\tilde b$ and $B_u=\tilde O(|\mathcal A|d^2\beta)$ from \eqref{eq:bonus-moment}, and with $v=f$ and $B_f=O(\beta)$ from \eqref{eq:conf-current}, we obtain
\[
\sum_{k,h}\Eb^{\pi^{(k)}}_\star\big[3H\tilde b^{(k)}_h+(H-h)f^{(k)}_h\big]\le\tilde O\Big(H^2|\mathcal A|d^{3/2}\sqrt{\beta K}\Big).
\]
Therefore,
\[
V^{\pi^\star}_{\star}-V^{\bar\pi}_{\star}\le2C|\mathcal A|(H-1)\tau+\tilde O\Big(H^2|\mathcal A|d^{3/2}\sqrt{\beta/K}\Big).
\]
With $\tau=\epsilon/(4C|\mathcal A|H)$, the first term is at most $\epsilon/2$. Since $\beta=c\log(2KH|\Theta|/\delta)$, choosing $K=\tilde O(H^4d^3|\mathcal A|^2\log(2|\Theta|/\delta)/\epsilon^2)$ with a sufficiently large constant makes the second term at most $\epsilon/2$. Each iteration collects $H-1$ trajectories in the true environment, so the total number of trajectories is $(H-1)K=\tilde O(H^5d^3|\mathcal A|^2\log(2|\Theta|/\delta)/\epsilon^2)$. All the above holds on $\mathcal E_1\cap\mathcal E_2$, which has probability at least $1-\delta$ by \Cref{lemma:good-event}.
\end{proof}

\section{UCB Analysis}\label{sec:ucb analysis}
In this section, we provide two lemmas: one verifies the key inequality that controls the propagated model error by the bonus, and the other connects the historical second moments with the uncertainty of the current learned policy under the true occupancy. We use the notation of \Cref{sec:proof of main thm}.

\begin{lemma}[Key inequality]\label{lemma:key-inequality}
On the event $\mathcal E_1$, for all $k\in[K]$, $h<H$, any policy $\pi$ and all $(s,a)$,
\[
\mathcal T_{\hat\theta_k,h}\langle\pi_{h+1},f^{(k)}_{h+1}\rangle(s,a)\le\tilde b^{(k)}_h(s,a).
\]
\end{lemma}

\begin{proof}
By the low-rank factorization of $\hat\theta_k$,
\[
\mathcal T_{\hat\theta_k,h}\langle\pi_{h+1},f^{(k)}_{h+1}\rangle(s,a)=\hat\phi^k_h(s,a)^\top w,\qquad w:=\int\hat\mu^k_h(s')\langle\pi_{h+1},f^{(k)}_{h+1}\rangle(s')\,ds',
\]
where $\hat\mu^k_h=\mu_{\hat\theta_k,h}$. Since $\langle\pi_{h+1},f^{(k)}_{h+1}\rangle\in[0,1]$, the normalization gives $\|w\|_2\le\sqrt d$. By the Cauchy--Schwarz inequality, $\hat\phi^k_h(s,a)^\top w\le\|\hat\phi^k_h(s,a)\|_{\hat\Lambda_{k,h}^{-1}}\|w\|_{\hat\Lambda_{k,h}}$, and
\begin{align*}
\|w\|^2_{\hat\Lambda_{k,h}}
&=\lambda\|w\|_2^2+\sum_{n<k}\Big(\mathop{\Eb}_{s'\sim\mathcal T_{\hat\theta_k,h}(\cdot|X_{n,h})}\big[\langle\pi_{h+1},f^{(k)}_{h+1}\rangle(s')\big]\Big)^2\\
&\overset{(a)}\le\lambda d+\sum_{n<k}\mathop{\Eb}_{s'\sim\mathcal T_{\hat\theta_k,h}(\cdot|X_{n,h})}\Big[\sum_{a'}\pi_{h+1}(a'|s')f^{(k)}_{h+1}(s',a')^2\Big]\\
&\overset{(b)}\le\lambda d+|\mathcal A|\sum_{n<k}\mathop{\Eb}_{s'\sim\mathcal T_{\hat\theta_k,h}(\cdot|X_{n,h}),\,a'\sim\mathtt u}\Big[f^{(k)}_{h+1}(s',a')^2\Big]\\
&\overset{(c)}\le\lambda d+c|\mathcal A|\beta,
\end{align*}
where $(a)$ follows from Jensen's inequality, $(b)$ is because the density ratio between $\pi_{h+1}(\cdot|s')$ and $\mathtt u$ is at most $|\mathcal A|$, and $(c)$ follows from \eqref{eq:conf-next}. Since $\alpha^2=c_0(\lambda d+|\mathcal A|\beta)$ with $c_0\ge\max\{1,c\}$, we have $\mathcal T_{\hat\theta_k,h}\langle\pi_{h+1},f^{(k)}_{h+1}\rangle(s,a)\le\alpha\|\hat\phi^k_h(s,a)\|_{\hat\Lambda_{k,h}^{-1}}$. In addition, the left-hand side is at most $1$ because $f^{(k)}_{h+1}\le1$. Taking the minimum of the two bounds yields the claim.
\end{proof}

The next lemma builds the connection between the historical second moments and the uncertainty of the current learned policy under the true occupancy.

\begin{lemma}[From historical second moments to the current policy]\label{lemma:summation}
Let $v^{(k)}_j:\mathcal S\times\mathcal A\to[0,1]$ ($k\in[K]$, $j\in[H]$) be either $\tilde b^{(k)}_j$ or $f^{(k)}_j$, so that $v^{(k)}_H\equiv0$, and suppose that $\sum_{n<k}\Eb_{m_{n,j}}[(v^{(k)}_j)^2]\le B_v$ for all $k$ and $j<H$, where $B_v\ge1$. Then
\[
\sum_{k=1}^K\sum_{j=1}^H\Eb^{\pi^{(k)}}_{\star}\big[v^{(k)}_j(s_j,a_j)\big]\lesssim1+\sqrt{|\mathcal A|B_vK}+H\sqrt{(d+|\mathcal A|B_v)\,dK\log(1+K/d)}.
\]
Moreover, on $\mathcal E_1$, $\sum_{k=1}^Ke_k\lesssim1+\sqrt{|\mathcal A|\beta K}$.
\end{lemma}

\begin{proof}
Define
\[
W_{k,j}=I+\sum_{n<k}\Eb^{\pi^{(n)}}_{\star}\big[\phi^\star_j(s_j,a_j)\phi^\star_j(s_j,a_j)^\top\big].
\]
\textbf{Step 1: $j\ge2$.} Given the previous state-action pair $x=(s_{j-1},a_{j-1})$, by the low-rank factorization of $\theta^\star$,
\[
\Eb^{\pi^{(k)}}_{\star}\big[v^{(k)}_j(s_j,a_j)\,\big|\,x\big]=\phi^\star_{j-1}(x)^\top w,\qquad w:=\int\mu^\star_{j-1}(s)\langle\pi^{(k)}_j,v^{(k)}_j\rangle(s)\,ds,
\]
with $\|w\|_2\le\sqrt d$. Moreover, since $s_j\sim d^{\pi^{(n)}}_{\theta^\star,j}$ under $\pi^{(n)}$,
\begin{align*}
\|w\|^2_{W_{k,j-1}}&=\|w\|_2^2+\sum_{n<k}\Eb^{\pi^{(n)}}_{\star}\Big[\big(\phi^\star_{j-1}(s_{j-1},a_{j-1})^\top w\big)^2\Big]\\
&\le d+\sum_{n<k}\Eb_{s\sim d^{\pi^{(n)}}_{\theta^\star,j}}\Big[\langle\pi^{(k)}_j,v^{(k)}_j\rangle(s)^2\Big]\le d+|\mathcal A|\sum_{n<k}\Eb_{m_{n,j}}\big[(v^{(k)}_j)^2\big]\le d+|\mathcal A|B_v,
\end{align*}
where we used Jensen's inequality and the density ratio between $\pi^{(k)}_j$ and $\mathtt u$. Hence, by the Cauchy--Schwarz inequality and $v^{(k)}_j\le1$,
\[
\Eb^{\pi^{(k)}}_{\star}\big[v^{(k)}_j\,\big|\,x\big]\le\sqrt{d+|\mathcal A|B_v}\,\min\big\{1,\|\phi^\star_{j-1}(x)\|_{W^{-1}_{k,j-1}}\big\}.
\]
Taking expectation and applying the Cauchy--Schwarz inequality over $k$,
\[
\sum_{k}\Eb^{\pi^{(k)}}_{\star}\big[v^{(k)}_j\big]\le\sqrt{(d+|\mathcal A|B_v)K}\sqrt{\sum_k\Eb^{\pi^{(k)}}_{\star}\Big[\min\big\{1,\|\phi^\star_{j-1}(s_{j-1},a_{j-1})\|^2_{W^{-1}_{k,j-1}}\big\}\Big]}.
\]
For a fixed step, $W_{k+1,j-1}=W_{k,j-1}+\Eb^{\pi^{(k)}}_{\star}[\phi^\star_{j-1}\phi^{\star\top}_{j-1}]$, $W_{k,j-1}\succeq I$ and $\|\phi^\star_{j-1}\|_2\le1$, so the eigenvalues of $W_{k,j-1}^{-1/2}\Eb^{\pi^{(k)}}_{\star}[\phi^\star_{j-1}\phi^{\star\top}_{j-1}]W_{k,j-1}^{-1/2}$ lie in $[0,1]$, and the potential-function argument (as in \Cref{lemma:clipped-elliptical}) gives
\[
\sum_k\Eb^{\pi^{(k)}}_{\star}\big[\|\phi^\star_{j-1}(s_{j-1},a_{j-1})\|^2_{W^{-1}_{k,j-1}}\big]\le2\log\det W_{K+1,j-1}\le2d\log(1+K/d).
\]
Summing over $j=2,\ldots,H$ gives the last term of the claim.

\textbf{Step 2: $j=1$.} Since the initial state is fixed, $m_{n,1}=\delta_{s_1}\times\mathtt u$ for all $n$, and the assumption gives $(k-1)\Eb_{a\sim\mathtt u}[v^{(k)}_1(s_1,a)^2]\le B_v$. Thus, for $k\ge2$,
\[
\Eb^{\pi^{(k)}}_{\star}\big[v^{(k)}_1(s_1,a_1)\big]\le\sqrt{\sum_a\pi^{(k)}_1(a|s_1)v^{(k)}_1(s_1,a)^2}\le\sqrt{|\mathcal A|\Eb_{a\sim\mathtt u}\big[v^{(k)}_1(s_1,a)^2\big]}\le\sqrt{\frac{|\mathcal A|B_v}{k-1}},
\]
and the term with $k=1$ is at most $1$. Summing over $k$ gives $1+O(\sqrt{|\mathcal A|B_vK})$.

Finally, by \eqref{eq:conf-current} with $j=1$, $(k-1)\Eb_{a\sim\mathtt u}[f^{(k)}_1(s_1,a)^2]\lesssim\beta$, so $e_k^2=\max_af^{(k)}_1(s_1,a)^2\le|\mathcal A|\Eb_{a\sim\mathtt u}[f^{(k)}_1(s_1,a)^2]\lesssim|\mathcal A|\beta/(k-1)$ for $k\ge2$, and $e_1\le1$. Summing over $k$ yields $\sum_ke_k\lesssim1+\sqrt{|\mathcal A|\beta K}$.
\end{proof}

\section{General MLE Analysis}\label{sec:general MLE analysis}

In this section, we present a general proposition that characterizes the performance of MLE on the windows. For a candidate model $\theta$ and a window start $X=(s_j,a_j)$, recall that the conditional likelihood of the window $Y=(s_{j+1},a_{j+1},s_{j+2})$ is
\[
p_\theta(Y|X)=\mathcal T_{\theta,j}(s_{j+1}|X)\,\mathtt u(a_{j+1})\,\mathcal T_{\theta,j+1}(s_{j+2}|s_{j+1},a_{j+1}),
\]
and write $p_\star=p_{\theta^\star}$. As in \Cref{lemma:TV and hellinger}, $\mathtt D^2_{\mathtt H}(P,Q)=\int(\sqrt P-\sqrt Q)^2=2(1-\int\sqrt{PQ})$.

\begin{proposition}\label{prop:window-mle}
Suppose $\mathcal L^{(k)}(\hat\theta_k)\le\min_{\theta\in\Theta}\mathcal L^{(k)}(\theta)+\beta$ for all $k$. Then, with probability at least $1-\delta/4$, for all $k\in[K]$,
\begin{align*}
\sum_{n<k}\sum_{j<H}\mathtt D^2_{\mathtt H}\big(p_{\hat\theta_k}(\cdot|X_{n,j}),p_\star(\cdot|X_{n,j})\big)\lesssim\beta,\qquad
\sum_{n<k}\sum_{j<H}\mathop{\Eb}_{X\sim m_{n,j}}\Big[\mathtt D^2_{\mathtt H}\big(p_{\hat\theta_k}(\cdot|X),p_\star(\cdot|X)\big)\Big]\lesssim\beta.
\end{align*}
Consequently, \eqref{eq:conf-next} and \eqref{eq:conf-current} hold for all $k\in[K]$ and $j<H$.
\end{proposition}

\begin{proof}
Fix a candidate $\theta$. For each window, let $\ell_{n,j}=\log\big(p_\theta(Y_{n,j}|X_{n,j})/p_\star(Y_{n,j}|X_{n,j})\big)$, let $z(X)=1-\int\sqrt{p_\theta(y|X)p_\star(y|X)}\,dy=\frac12\mathtt D^2_{\mathtt H}(p_\theta(\cdot|X),p_\star(\cdot|X))$ be the squared Hellinger distance given $X$, and let $\bar z_{n,j}=\Eb_{X\sim m_{n,j}}[z(X)]$ be its average before the window is sampled. Given $X$, $\Eb[\exp(\ell_{n,j}/2)|X_{n,j}=X]=1-z(X)\le\exp(-z(X))$, and hence
\[
\Eb\big[\exp\big(\ell_{n,j}/2+z(X_{n,j})\big)\,\big|\,X_{n,j}\big]\le1.
\]
Since $m_{n,j}$ is determined by $\pi^{(n)}$, which is fixed before the window is sampled, the same argument with the expectation also taken over $X_{n,j}\sim m_{n,j}$ gives $\Eb[\exp(\ell_{n,j}/2+\bar z_{n,j})\,|\,\mathcal F_{n,j}]\le1$, where $\mathcal F_{n,j}$ is the history before the window $(n,j)$ is sampled. Multiplying these inequalities in the sampling order of the windows and applying the tower property, we have
\[
\Eb\Big[\exp\Big(\sum_{n<k,j<H}\big(\ell_{n,j}/2+z(X_{n,j})\big)\Big)\Big]\le1,\qquad
\Eb\Big[\exp\Big(\sum_{n<k,j<H}\big(\ell_{n,j}/2+\bar z_{n,j}\big)\Big)\Big]\le1.
\]
By Markov's inequality and a union bound over $\theta\in\Theta$, $k\in[K]$ and the two quantities, with probability at least $1-\delta/4$, for all $\theta$ and $k$,
\[
\sum_{n<k,j<H}z(X_{n,j})\le-\frac12\sum_{n<k,j<H}\ell_{n,j}+\log\frac{8K|\Theta|}{\delta},\qquad\sum_{n<k,j<H}\bar z_{n,j}\le-\frac12\sum_{n<k,j<H}\ell_{n,j}+\log\frac{8K|\Theta|}{\delta}.
\]
Now take $\theta=\hat\theta_k$. Since $\theta^\star\in\Theta$, the approximate MLE condition gives $\sum_{n<k,j<H}\ell_{n,j}=\mathcal L^{(k)}(\theta^\star)-\mathcal L^{(k)}(\hat\theta_k)\ge-\beta$. Therefore, both historical sums are at most $\beta/2+\log(8K|\Theta|/\delta)=O(\beta)$, which proves the two displayed bounds.

For \eqref{eq:conf-next}, note that under $p_\theta(\cdot|X)$, the pair $(s_{j+1},a_{j+1})$ follows $\mathcal T_{\theta,j}(\cdot|X)\times\mathtt u$ and $s_{j+2}$ follows $\mathcal T_{\theta,j+1}(\cdot|s_{j+1},a_{j+1})$. By the conditional Hellinger inequality in \Cref{lemma:TV and hellinger}, with the marginal taken under $\hat\theta_k$,
\[
\mathop{\Eb}_{s'\sim\mathcal T_{\hat\theta_k,j}(\cdot|X),\,a'\sim\mathtt u}\Big[\mathtt D^2_{\mathtt H}\big(\mathcal T_{\hat\theta_k,j+1}(\cdot|s',a'),\mathcal T_{\theta^\star,j+1}(\cdot|s',a')\big)\Big]\le8\,\mathtt D^2_{\mathtt H}\big(p_{\hat\theta_k}(\cdot|X),p_\star(\cdot|X)\big).
\]
By the first inequality of \Cref{lemma:TV and hellinger}, $f^{(k)}_{j+1}(s',a')^2\le2\mathtt D^2_{\mathtt H}(\mathcal T_{\hat\theta_k,j+1}(\cdot|s',a'),\mathcal T_{\theta^\star,j+1}(\cdot|s',a'))$. Taking $X=X_{n,j}$ and summing over $n<k$ yield \eqref{eq:conf-next}. For \eqref{eq:conf-current}, marginalizing the window to $s_{j+1}$ does not increase the Hellinger distance, so $f^{(k)}_j(X)^2\le2\mathtt D^2_{\mathtt H}(\mathcal T_{\hat\theta_k,j}(\cdot|X),\mathcal T_{\theta^\star,j}(\cdot|X))\le2\mathtt D^2_{\mathtt H}(p_{\hat\theta_k}(\cdot|X),p_\star(\cdot|X))$. Taking expectation over $X\sim m_{n,j}$ and summing over $n<k$ yield \eqref{eq:conf-current}.
\end{proof}

\noindent\textbf{Remark (stagewise model classes).}
When the parameters are separable across stages, i.e., $\Theta=\Theta_1\times\cdots\times\Theta_H$ with $\mathcal T_{\theta,h}$ depending only on $\theta_h$, the observations of each transition at its two window positions can be pooled stage by stage. In this case, one controls the one-sided log-likelihood advantage separately at the two window positions where each transition appears, and takes a union bound over candidate pairs of adjacent stages; this replaces $\log|\Theta|$ by the logarithm of the stagewise class sizes.

\section{Existence of approximate low-rank factorization}\label{sec:RBF proof}

\begin{theorem}[]
\label{thm:fourier-factorization-constructive}
Let $\mathcal{X}\subset\mathbb{R}^{D_x}$ and $\mathcal{Y}\subset\mathbb{R}^{D}$ be compact sets.
Let $\Theta$ be a finite model class. For each $(\theta,x)\in\Theta\times\mathcal{X}$,
assume the conditional distribution $P_\theta(\cdot\mid x)$ admits a density $P_\theta(y\mid x)$ on $\mathcal{Y}$.
We assume access to a sampling oracle that provides i.i.d.\ samples $y^{(1)},\dots,y^{(N)}\sim P_\theta(\cdot\mid x)$.

\paragraph{(A1) Boundedness.}
Assume there exists $M > 0$ such that $\sup_{\theta,x,y} P_\theta(y\mid x)\le M$.

\paragraph{(A2) Smoothness and Boundary Vanishing.}
Assume there exists an integer $m>D$ and a constant $C_m$ such that $P_\theta(\cdot\mid x)$ is $m$-times differentiable on $\mathcal{Y}$.
Crucially, assume that $P_\theta(\cdot\mid x)$ and its derivatives up to order $m-1$ vanish on the boundary $\partial\mathcal{Y}$.

Then, for any radius $W \ge 1$, there exists a randomized algorithm that constructs feature maps $\hat\phi_\theta(x), \mu(y) \in \mathbb{R}^{2d}$ using $N$ samples, such that with probability at least $1-\delta$, uniformly for all $(\theta, x, y) \in \Theta \times \mathcal{X} \times \mathcal{Y}_0$:
\[
\bigl|P_\theta(y\mid x)-\hat\phi_\theta(x)^\top \mu(y)\bigr|
\ \le\
C_{\rm trunc} W^{D-m}
\ +\
C_{\rm rand} \Vol(B_W) \left( \sqrt{\frac{1}{d}} + \sqrt{\frac{1}{N}} \right),
\]
where $C_{\rm trunc}$ and $C_{\rm rand}$ are constants depending on $m, D, \Theta, \mathcal{X}, \delta$ (explicit form in proof).
\end{theorem}

\begin{proof}
We first construct the algorithm explicitly and then provide the error analysis to prove the bound.

\paragraph{Construction of the Algorithm.}
The algorithm proceeds as follows:
\begin{enumerate}
    \item \textbf{Frequency Sampling:} Let $B_W = \{w \in \mathbb{R}^D : \|w\| \le W\}$. Draw $d$ i.i.d. frequencies $w_1, \dots, w_d$ uniformly from $B_W$. Let $\Vol(B_W)$ denote the volume of this ball.
    \item \textbf{Target Features ($\mu$):} Construct the feature map $\mu(y) \in \mathbb{R}^{2d}$ as:
    \[
    \mu(y) = \frac{1}{\sqrt{d}} \left( \cos(2\pi w_1^\top y), -\sin(2\pi w_1^\top y), \dots, \cos(2\pi w_d^\top y), -\sin(2\pi w_d^\top y) \right).
    \]
    \emph{(Note the negative sign on the sine terms, corresponding to the inverse Fourier transform.)}
    \item \textbf{Context Features ($\hat\phi$):} Draw $N$ samples $y^{(1)}, \dots, y^{(N)} \sim P_\theta(\cdot \mid x)$. For each frequency $k \in \{1,\dots,d\}$, compute the empirical Fourier coefficient:
    \[
    \hat g_k = \frac{1}{N} \sum_{n=1}^N e^{-2\pi i w_k^\top y^{(n)}}.
    \]
    Let $\hat g_k^{\rm r} = \Re(\hat g_k)$ and $\hat g_k^{\rm i} = \Im(\hat g_k)$. Construct $\hat\phi_\theta(x) \in \mathbb{R}^{2d}$ as:
    \[
    \hat\phi_\theta(x) = \frac{\Vol(B_W)}{\sqrt{d}} \left( \hat g_1^{\rm r}, \hat g_1^{\rm i}, \dots, \hat g_d^{\rm r}, \hat g_d^{\rm i} \right).
    \]
\end{enumerate}

\paragraph{Analysis.}
We decompose the error into three terms: Bias (Truncation), Variance (Random Features), and Estimation Error (Finite Samples).
\[
|P_\theta - \hat\phi^\top \mu| \le \underbrace{|P_\theta - \bar P_\theta|}_{\text{(I)}} + \underbrace{|\bar P_\theta - \phi^\top \mu|}_{\text{(II)}} + \underbrace{|\phi^\top \mu - \hat\phi^\top \mu|}_{\text{(III)}}.
\]

\textbf{Step 1: Truncation Error (I).}
Let $g_\theta(w)$ be the Fourier transform of $P_\theta$ (extended by zero to $\mathbb{R}^D$).
Due to Assumption (A2) (boundary vanishing), integration by parts yields the spectral decay $|g_\theta(w)| \le c_{D,m}C_m \|w\|^{-m}$.
The reconstruction $\bar P_\theta$ integrates frequencies only up to $W$. The tail mass is:
\[
\text{(I)} \le \int_{\|w\|>W} |g_\theta(w)| dw \le \frac{S_{D-1} c_{D,m} C_m}{m-D} W^{D-m}.
\]

\textbf{Step 2: Random Feature Error (II).}
The term $\phi^\top \mu$ is a Monte Carlo approximation of the integral $\bar P_\theta = \int_{B_W} g_\theta(w) e^{2\pi i w^\top y} dw$.
The variable $Z(w) = \Vol(B_W) \Re(g_\theta(w) e^{2\pi i w^\top y})$ is bounded by $2\Vol(B_W)$.
Applying Hoeffding's inequality over $d$ features (and union bounding over the domain) gives:
\[
\text{(II)} \le \Vol(B_W) \sqrt{\frac{2 \log(4|\Theta||\mathcal{X}||\mathcal{Y}_0|/\delta)}{d}}.
\]

\textbf{Step 3: Estimation Error (III).}
This term arises from approximating $g_\theta(w_k)$ with $\hat g_k$ using $N$ samples.
\[
\text{(III)} = \left| \sum_{k=1}^d \frac{\Vol(B_W)}{d} (\hat g_k - g_k) \cdot (\dots) \right|.
\]
Standard concentration (Hoeffding) on the $N$ samples shows that for any fixed $w_k$, $|\hat g_k - g_k| \lesssim \sqrt{1/N}$.
Averaging these errors over $d$ frequencies does not reduce the variance below $O(1/\sqrt{N})$ because the samples are shared.
The explicit bound (with union bounds) is:
\[
\text{(III)} \le \Vol(B_W) \sqrt{\frac{2\log(4|\Theta||\mathcal{X}|d/\delta)}{N}}.
\]

\paragraph{Conclusion.}
Summing the terms yields the stated bound.
\end{proof}

\section{Auxillary Lemmas}
We first establish one of the basic lemmas in RL literature: value difference lemma, which upper bound the difference between two values under different models, rewards, and policies.

\begin{lemma}[Value-difference lemma]\label{lemma:value difference}
Let $\theta,\theta'\in\Theta$, policies $\pi,\pi'$, and nonnegative reward functions $r,r'$.
Assume the value functions are uniformly bounded:
\[
0\le V_{\theta,r}^{\pi}(s)\le B_r,\qquad 0\le V_{\theta',r'}^{\pi'}(s)\le B_{r'}
\quad\text{for all }s.
\]
Then the value difference admits the following two upper bounds:
\begin{align}
V_{\theta,r}^{\pi} - V_{\theta',r'}^{\pi'}
\le\;&
V_{\theta,r-r'}^{\pi}
+ \sum_{h=1}^H \Eb_{\theta,\pi}\!\left[
\sum_{a} Q_{\theta',r',h}^{\pi'}(s_h,a)\bigl(\pi(a|s_h)-\pi'(a|s_h)\bigr)
\right] \nonumber\\
&\qquad
+ B_{r'} \sum_{h=1}^H \Eb_{\theta,\pi}\!\left[
\mathtt{D}_{\mathtt{TV}}\!\left(\mathcal{T}_{\theta,h}(\cdot|s_h,a_h),\mathcal{T}_{\theta',h}(\cdot|s_h,a_h)\right)
\right], \label{eq:vdl-1}\\[4pt]
V_{\theta,r}^{\pi} - V_{\theta',r'}^{\pi'}
\le\;&
V_{\theta',r-r'}^{\pi'}
+ \sum_{h=1}^H \Eb_{\theta',\pi'}\!\left[
\sum_{a} Q_{\theta,r,h}^{\pi}(s_h,a)\bigl(\pi(a|s_h)-\pi'(a|s_h)\bigr)
\right] \nonumber\\
&\qquad
+ B_{r} \sum_{h=1}^H \Eb_{\theta',\pi'}\!\left[
\mathtt{D}_{\mathtt{TV}}\!\left(\mathcal{T}_{\theta,h}(\cdot|s_h,a_h),\mathcal{T}_{\theta',h}(\cdot|s_h,a_h)\right)
\right]. \label{eq:vdl-2}
\end{align}
\end{lemma}

\begin{proof}
We prove \eqref{eq:vdl-1}; \eqref{eq:vdl-2} follows by swapping $(\theta,\pi,r)$ and
$(\theta',\pi',r')$.

For any $h$ and state $s$, recall $V_{\theta,r,h}^{\pi}(s)=\sum_a \pi(a|s)Q_{\theta,r,h}^{\pi}(s,a)$.
Thus
\begin{align*}
V_{\theta,r,1}^{\pi}(s_1)-V_{\theta',r',1}^{\pi'}(s_1)
&= \sum_a \pi(a|s_1)Q_{\theta,r,1}^{\pi}(s_1,a) - \sum_a \pi'(a|s_1)Q_{\theta',r',1}^{\pi'}(s_1,a)\\
&= \sum_a\bigl(\pi(a|s_1)-\pi'(a|s_1)\bigr)Q_{\theta',r',1}^{\pi'}(s_1,a)
  + \sum_a \pi(a|s_1)\bigl(Q_{\theta,r,1}^{\pi}-Q_{\theta',r',1}^{\pi'}\bigr)(s_1,a).
\end{align*}
Expand the $Q$-difference using Bellman equations:
\[
Q_{\theta,r,1}^{\pi}(s_1,a)=r(s_1,a)+\Eb_{s'\sim \mathcal{T}_{\theta,1}(\cdot|s_1,a)}\!\left[V_{\theta,r,2}^{\pi}(s')\right],
\]
and similarly for $(\theta',r',\pi')$. This yields
\begin{align*}
&\sum_a \pi(a|s_1)\bigl(Q_{\theta,r,1}^{\pi}-Q_{\theta',r',1}^{\pi'}\bigr)(s_1,a)\\
&=
\sum_a \pi(a|s_1)\bigl(r(s_1,a)-r'(s_1,a)\bigr)
+ \sum_a \pi(a|s_1)\Bigl(\mathcal{T}_{\theta,1}V_{\theta,r,2}^{\pi}-\mathcal{T}_{\theta',1}V_{\theta',r',2}^{\pi'}\Bigr)(s_1,a)\\
&=
\sum_a \pi(a|s_1)\bigl(r(s_1,a)-r'(s_1,a)\bigr)
+ \sum_a \pi(a|s_1)\mathcal{T}_{\theta,1}\!\bigl(V_{\theta,r,2}^{\pi}-V_{\theta',r',2}^{\pi'}\bigr)(s_1,a)\\
&\qquad
+ \sum_a \pi(a|s_1)\bigl(\mathcal{T}_{\theta,1}-\mathcal{T}_{\theta',1}\bigr)V_{\theta',r',2}^{\pi'}(s_1,a).
\end{align*}
For the last term, use the standard TV inequality: for any bounded $u$ with $0\le u\le B_{r'}$,
\[
\left| \Eb_{P}[u]-\Eb_{Q}[u]\right|
\le B_{r'}\,\mathtt{D}_{\mathtt{TV}}(P,Q).
\]
Since rewards are nonnegative, $V_{\theta',r',2}^{\pi'}\in[0,B_{r'}]$, hence
\[
\bigl(\mathcal{T}_{\theta,1}-\mathcal{T}_{\theta',1}\bigr)V_{\theta',r',2}^{\pi'}(s_1,a)
\le B_{r'}\,\mathtt{D}_{\mathtt{TV}}\!\left(\mathcal{T}_{\theta,1}(\cdot|s_1,a),\mathcal{T}_{\theta',1}(\cdot|s_1,a)\right).
\]
Combining the above displays and taking expectation over $(s_1,a_1)\sim(\theta,\pi)$ gives
\begin{align*}
V_{\theta,r,1}^{\pi}(s_1)-V_{\theta',r',1}^{\pi'}(s_1)
&\le
\Eb_{\theta,\pi}\!\left[r(s_1,a_1)-r'(s_1,a_1)\right]
+ \Eb_{\theta,\pi}\!\left[V_{\theta,r,2}^{\pi}(s_2)-V_{\theta',r',2}^{\pi'}(s_2)\right]\\
&\quad+
\Eb_{\theta,\pi}\!\left[\sum_a Q_{\theta',r',1}^{\pi'}(s_1,a)\bigl(\pi(a|s_1)-\pi'(a|s_1)\bigr)\right]\\
&\quad+
B_{r'}\,\Eb_{\theta,\pi}\!\left[\mathtt{D}_{\mathtt{TV}}\!\left(\mathcal{T}_{\theta,1}(\cdot|s_1,a_1),\mathcal{T}_{\theta',1}(\cdot|s_1,a_1)\right)\right].
\end{align*}
Iterating this inequality from $h=1$ to $H$ (telescoping the value terms) yields
\eqref{eq:vdl-1}, noting that $V_{\theta,r,H+1}^{\pi}\equiv V_{\theta',r',H+1}^{\pi'}\equiv 0$.
\end{proof}

The following lemma characterize the relationship between the total variation distance and the Hellinger-squared distance. Note that the result for probability measures has been proved in Lemma H.1 in 
Since we consider more general bounded measures, we provide the full proof for completeness.
\begin{lemma}\label{lemma:TV and hellinger}
Given two bounded measures $P$ and $Q$ defined on the set $\mathcal{X}$. Let $|P| = \sum_{x\in\mathcal{X}} P(x)$ and $|Q| = \sum_{x\in\mathcal{X}}Q(x).$ We have
\[ \mathtt{D}_{\mathtt{TV}}^2(P,Q)  \leq 4(|P|+|Q|)\mathtt{D}_{\mathtt{H}}^2(P,Q)  \]
In addition, if $P_{Y|X}, Q_{Y|X}$ are two conditional distributions over a random variable $Y$, and $P_{X,Y} = P_{Y|X}P$, $Q_{X,Y}= Q_{Y|X}Q$ are the joint distributions when $X$ follows the distributions $P$ and $Q$, respectively, we have
\[\mathop{\Eb}_{X\sim P }\left[ \mathtt{D}_{\mathtt{H}}^2(P_{Y|X},Q_{Y|X})\right] \leq 8\mathtt{D}_{\mathtt{H}}^2 (P_{X,Y},Q_{X,Y}). \]
\end{lemma}

\begin{proof}
We first prove the first inequality.  By the definition of total variation distance, we have
    \begin{align*}
        \mathtt{D}_{\mathtt{TV}}^2(P,Q) & = \left(\sum_x |P(x) - Q(x)| \right)^2\\
        & = \left(\sum_x \left(\sqrt{P(x)} - \sqrt{Q(x)}\right)\left(\sqrt{P(x)} + \sqrt{Q(x)} \right) \right)^2\\
        &\overset{(a)}\leq \left(\sum_x\left(\sqrt{P(x)} - \sqrt{Q(x)}\right)^2\right) \left( 2\sum_{x}\left( P(x) + Q(x)\right) \right)\\
        &\leq 4(|P| + |Q|) \mathtt{D}_{\mathtt{H}}^2(P,Q),
    \end{align*}
where $(a)$ follows from the Cauchy's inequality and because $(a+b)^2\leq 2a^2+2b^2$.

For the second inequality, we have,
\begin{align*}
    \mathop{\Eb}_{X\sim P } &\left[ \mathtt{D}_{\mathtt{H}}^2(P_{Y|X},Q_{Y|X})\right]\\
    &= \sum_{x}P(x)\left( \sum_y \left(\sqrt{P_{Y|X}(y)} - \sqrt{Q_{Y|X}(y)} \right)^2 \right) \\
    & = \sum_{x,y} \left( \sqrt{P_{X,Y}(x,y)} - \sqrt{Q_{X,Y}(x,y)} + \sqrt{Q_{Y|X}(y) Q(x)} - \sqrt{Q_{Y|X}(y) P(x)}  \right)^2 \\
    &\leq 2\sum_{x,y} \left( \sqrt{P_{X,Y}(x,y)} - \sqrt{Q_{X,Y}(x,y)} \right)^2 + 2\sum_{x,y} Q_{Y|X}(y)\left( \sqrt{  Q(x)} - \sqrt{  P(x)}  \right)^2 \\
    & = 4\mathtt{D}_{\mathtt{H}}^2(P_{X,Y},Q_{X,Y}) + 2(|P|+|Q| - 2\sum_x\sqrt{P(x)Q(x)}) \\
    &\overset{(a)}\leq 4\mathtt{D}_{\mathtt{H}}^2(P_{X,Y},Q_{X,Y}) + 2(|P|+|Q| - 2\sum_x\sum_y\sqrt{P_{Y|X}(y)P(x)Q_{Y|X}(y)Q(x)}) \\
    & = 8 \mathtt{D}_{\mathtt{H}}^2(P_{X,Y},Q_{X,Y}),
\end{align*}
where $(a)$ follows from the Cauchy's inequality that applies on $\sum_y\sqrt{P_{Y|X}(y)Q_{Y|X}(y)}$.
\end{proof}

\begin{lemma}[Elliptical potential lemma]
\label{lemma:clipped-elliptical}
Let $(\mathcal{F}_i)_{i\ge 0}$ be a filtration. For $i\ge 1$, let $Y_i\in\mathbb{R}^d$ be an $\mathcal{F}_i$-measurable random vector such that its conditional law given the past is allowed to be adaptive, i.e.,
\[
Y_i \mid \mathcal{F}_{i-1} \sim P_i(\cdot \mid X_i),
\]
where $X_i$ and the kernel $P_i(\cdot\mid X_i)$ are $\mathcal{F}_{i-1}$-measurable (in particular, they may be arbitrary functions of the history).
Fix $\lambda>0$ and define the (predictable) Gram matrices
\[
A_1 := \lambda I_d,
\qquad
A_i := \lambda I_d + \sum_{j<i} Y_j Y_j^\top \ \ (i\ge 2).
\]
Let $f_i := \|Y_i\|_{A_i^{-1}}^2 = Y_i^\top A_i^{-1} Y_i$ and its clipped version
\[
\tilde f_i := \min\{1,f_i\}.
\]
Then the following hold:

\medskip
\noindent\textbf{(i) Conditional one-step bound.} For every $i\ge 1$,
\[
\mathbb{E}\!\left[\tilde f_i \mid \mathcal{F}_{i-1}\right]
\;\le\;
2\,\mathbb{E}\!\left[\log\!\left(\frac{\det(A_{i+1})}{\det(A_i)}\right)\Bigm|\mathcal{F}_{i-1}\right].
\]

\medskip
\noindent\textbf{(ii) Telescoping (expectation) bound.} Consequently,
\[
\mathbb{E}\!\left[\sum_{i=1}^n \tilde f_i\right]
\;\le\;
2\,\mathbb{E}\!\left[\log\!\left(\frac{\det(A_{n+1})}{\det(\lambda I_d)}\right)\right].
\]

\medskip
\noindent\textbf{(iii) Deterministic upper bound under bounded norms.}
If additionally $\|Y_i\|\le L$ almost surely for all $i$, then
\[
\sum_{i=1}^n \tilde f_i
\;\le\;
2\log\!\left(\frac{\det(A_{n+1})}{\det(\lambda I_d)}\right)
\;\le\;
2d\log\!\left(1+\frac{nL^2}{\lambda d}\right),
\]
and hence also
\[
\mathbb{E}\!\left[\sum_{i=1}^n \tilde f_i\right]
\;\le\;
2d\log\!\left(1+\frac{nL^2}{\lambda d}\right).
\]
\end{lemma}

\begin{proof}
We use two standard facts.

\medskip
\noindent\textbf{Step 1: Clipping versus log.}
For all $x\ge 0$,
\[
\min\{1,x\}\le 2\log(1+x).
\]
Indeed, if $x\in[0,1]$ then $\log(1+x)\ge x/2$; if $x\ge 1$ then
$2\log(1+x)\ge 2\log 2 \ge 1$.

\medskip
\noindent\textbf{Step 2: Matrix determinant lemma.}
By $A_{i+1}=A_i+Y_iY_i^\top$ and the matrix determinant lemma,
\[
\det(A_{i+1})=\det(A_i)\bigl(1+Y_i^\top A_i^{-1}Y_i\bigr)
=\det(A_i)\bigl(1+f_i\bigr),
\]
hence
\[
\log\!\left(\frac{\det(A_{i+1})}{\det(A_i)}\right)=\log(1+f_i).
\]

\medskip
\noindent\textbf{Proof of (i).}
Condition on $\mathcal{F}_{i-1}$. The matrix $A_i$ is $\mathcal{F}_{i-1}$-measurable by construction, while $Y_i$ is drawn conditionally from an arbitrary kernel measurable w.r.t.\ $\mathcal{F}_{i-1}$.
Applying Step~1 with $x=f_i$ and then Step~2 gives
\[
\tilde f_i=\min\{1,f_i\}\le 2\log(1+f_i)
=2\log\!\left(\frac{\det(A_{i+1})}{\det(A_i)}\right).
\]
Taking conditional expectations w.r.t.\ $\mathcal{F}_{i-1}$ yields (i).

\medskip
\noindent\textbf{Proof of (ii).}
Sum the inequality in (i) over $i=1,\dots,n$ and use telescoping:
\[
\sum_{i=1}^n \log\!\left(\frac{\det(A_{i+1})}{\det(A_i)}\right)
=
\log\!\left(\frac{\det(A_{n+1})}{\det(A_1)}\right)
=
\log\!\left(\frac{\det(A_{n+1})}{\det(\lambda I_d)}\right).
\]
Finally take total expectation.

\medskip
\noindent\textbf{Proof of (iii).}
The first inequality is Step~1 + Step~2 summed over $i$ (no expectation needed).
For the last inequality, note that $A_{n+1}\preceq \lambda I_d+\sum_{i=1}^n L^2 I_d
= (\lambda+nL^2)I_d$, so
\[
\det(A_{n+1})\le (\lambda+nL^2)^d
\quad\Rightarrow\quad
\log\!\left(\frac{\det(A_{n+1})}{\det(\lambda I_d)}\right)
\le d\log\!\left(1+\frac{nL^2}{\lambda}\right)
\le d\log\!\left(1+\frac{nL^2}{\lambda d}\,d\right),
\]
which implies the displayed $2d\log(1+\frac{nL^2}{\lambda d})$ bound (a slightly sharper bound follows from AM--GM on eigenvalues).
\end{proof}

\begin{lemma}\label{lemma:bernstein on bonus}
Let $\Phi$ be a finite class of feature maps $\phi:\mathcal X \to \mathbb R^d$
satisfying $\|\phi(x)\|_2 \le 1$ for all $x \in \mathcal X$.
Let $\{x_n\}_{n=1}^K$ be a sequence of random variables such that
\[
x_n \sim P_n(\cdot \mid \mathcal F_{n-1}),
\]
where $\{\mathcal F_n\}$ is a filtration.
Define
\[
\Lambda_k
=
\lambda I
+
\sum_{n=1}^k
\phi(x_n)\phi(x_n)^\top.
\]
Then with probability at least $1-\delta$, for all $\phi \in \Phi$ and all $k \le K$,
\[
\sum_{n=1}^k
\mathbb E_{X_n \sim P_n}
\!\left[
\min\!\left\{1,\,
\|\phi(X_n)\|_{\Lambda_k^{-1}}^2
\right\}
\right]
\le
\sum_{n=1}^k
\min\!\left\{1,\,
\|\phi(x_n)\|_{\Lambda_k^{-1}}^2
\right\}
+
O\!\left(
\log \frac{K|\Phi|}{\delta}
+
d^2 \log \frac{K}{\lambda}
\right).
\]

Furthermore, the RHS can be simplied to $\tilde{O}\left(d^2\log (K|\Phi|/\delta)\right)$.

\end{lemma}

\begin{proof}
Since $\Lambda_k \succeq \lambda I$, we have
$
0 \preceq \Lambda_k^{-1} \preceq \tfrac{1}{\lambda} I.
$
Let $\mathcal M$ be an $\epsilon$-net (under Frobenius norm) of the set
\[
\left\{
M \succeq 0 : \|M\|_{\mathrm{op}} \le \tfrac{1}{\lambda}
\right\}.
\]
Standard covering arguments imply
\[
\log |\mathcal M|
=
O\!\left(d^2 \log \tfrac{1}{\lambda\epsilon}\right).
\]

Fix $\phi \in \Phi$, $M \in \mathcal M$, and $k \le K$.
Define
\[
Y_n
=
\min\!\left\{1,\,
\|\phi(x_n)\|_M^2
\right\},
\qquad
\mu_n
=
\mathbb E_{X_n \sim P_n}[Y_n].
\]
Then $0 \le Y_n \le 1$, and
\[
\mathrm{Var}(Y_n) \le \mu_n.
\]
By Bernstein's inequality for conditionally independent bounded random variables,
with probability at least $1-\delta$,
\[
\sum_{n=1}^k (\mu_n - Y_n)
\le
\sqrt{2 \log(1/\delta) \sum_{n=1}^k \mu_n}
+
\tfrac{2}{3} \log(1/\delta).
\]
Rearranging yields
\[
\sum_{n=1}^k \mu_n
\le
\sum_{n=1}^k Y_n
+
O\!\left(\log \tfrac{1}{\delta}\right).
\]

Applying a union bound over $\phi \in \Phi$, $M \in \mathcal M$, and $k \le K$,
we obtain that with probability at least $1-\delta$,
\[
\sum_{n=1}^k
\mathbb E
\!\left[
\min\{1,\|\phi(X_n)\|_M^2\}
\right]
\le
\sum_{n=1}^k
\min\{1,\|\phi(x_n)\|_M^2\}
+
O\!\left(
\log \tfrac{K|\Phi||\mathcal M|}{\delta}
\right).
\]

By construction of $\mathcal M$, for each $k$ there exists
$M \in \mathcal M$ such that
\[
\|M - \Lambda_k^{-1}\|_F \le \epsilon.
\]
Using $\|\phi(x)\|_2 \le 1$, we have
\[
\bigl|
\|\phi(x)\|_M^2
-
\|\phi(x)\|_{\Lambda_k^{-1}}^2
\bigr|
\le
\epsilon.
\]
Since the function $a \mapsto \min\{1,a\}$ is $1$-Lipschitz,
\[
\bigl|
\min\{1,\|\phi(x)\|_M^2\}
-
\min\{1,\|\phi(x)\|_{\Lambda_k^{-1}}^2\}
\bigr|
\le
\epsilon.
\]
Therefore,
\[
\sum_{n=1}^k
\mathbb E
\!\left[
\min\{1,\|\phi(X_n)\|_{\Lambda_k^{-1}}^2\}
\right]
\le
\sum_{n=1}^k
\min\{1,\|\phi(x_n)\|_{\Lambda_k^{-1}}^2\}
+
O\!\left(
\log \tfrac{K|\Phi||\mathcal M|}{\delta}
+
k\epsilon
\right).
\]

Choosing $\epsilon = 1/K$ ensures $k\epsilon \le 1$ and
\[
\log |\mathcal M|
=
O\!\left(d^2 \log \tfrac{K}{\lambda}\right),
\]
which completes the proof.
\end{proof}

\section{Experiment Details}\label{sec:more experiment results}
We provide more details and training results in this section.
\begin{figure*}[ht]
  \centering

  \begin{subfigure}[t]{0.32\textwidth}
    \centering
    \includegraphics[width=\linewidth]{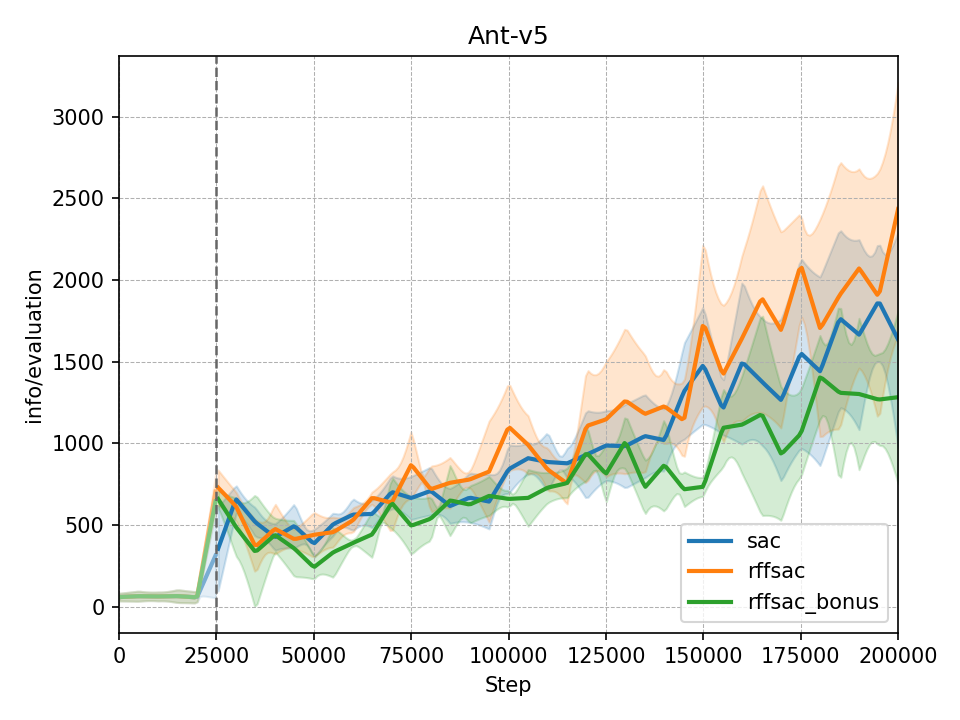}
    \label{fig:ablation_a}
  \end{subfigure}
  \hfill
  \begin{subfigure}[t]{0.32\textwidth}
    \centering
    \includegraphics[width=\linewidth]{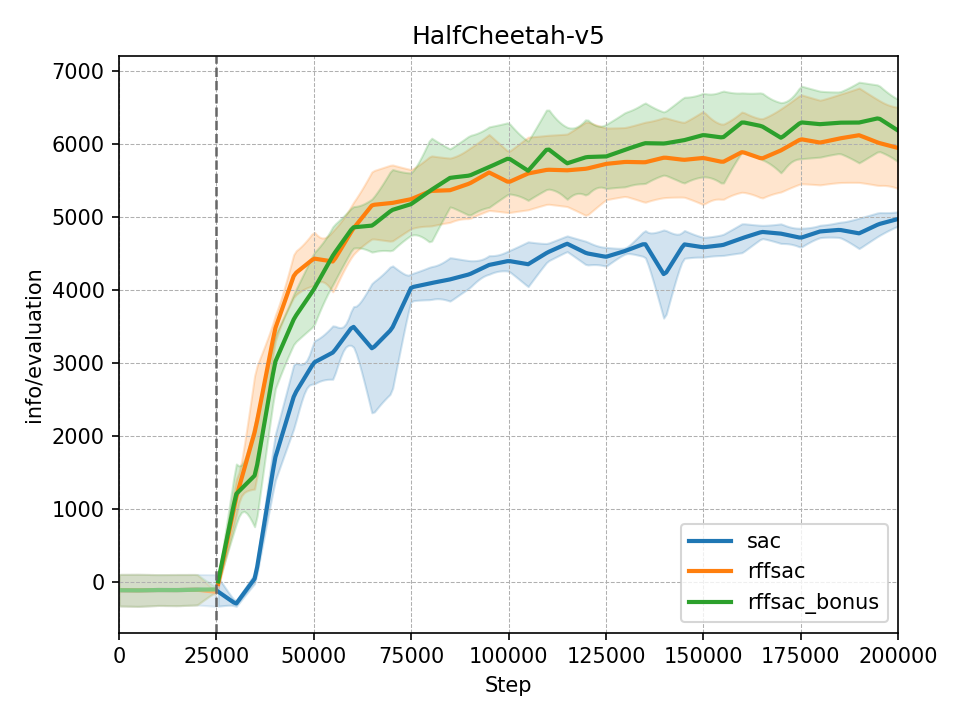}
    \label{fig:HalfCheetah}
  \end{subfigure}
  \hfill
  \begin{subfigure}[t]{0.32\textwidth}
    \centering
    \includegraphics[width=\linewidth]{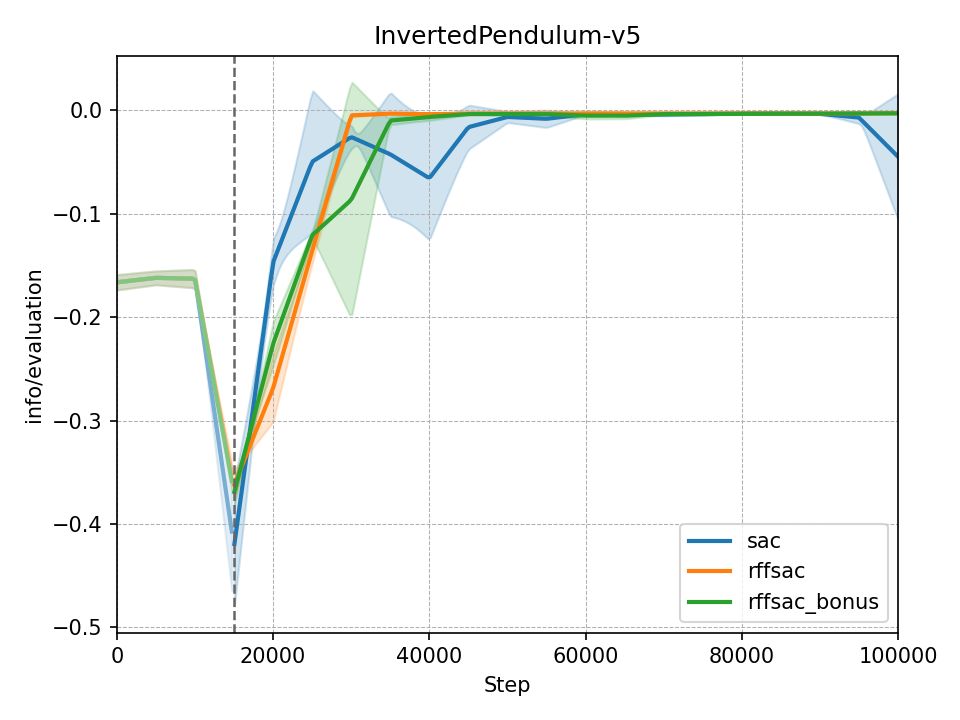}
    \label{fig:ablation_c}
  \end{subfigure}

  \vspace{0.5em}

  \begin{subfigure}[t]{0.32\textwidth}
    \centering
    \includegraphics[width=\linewidth]{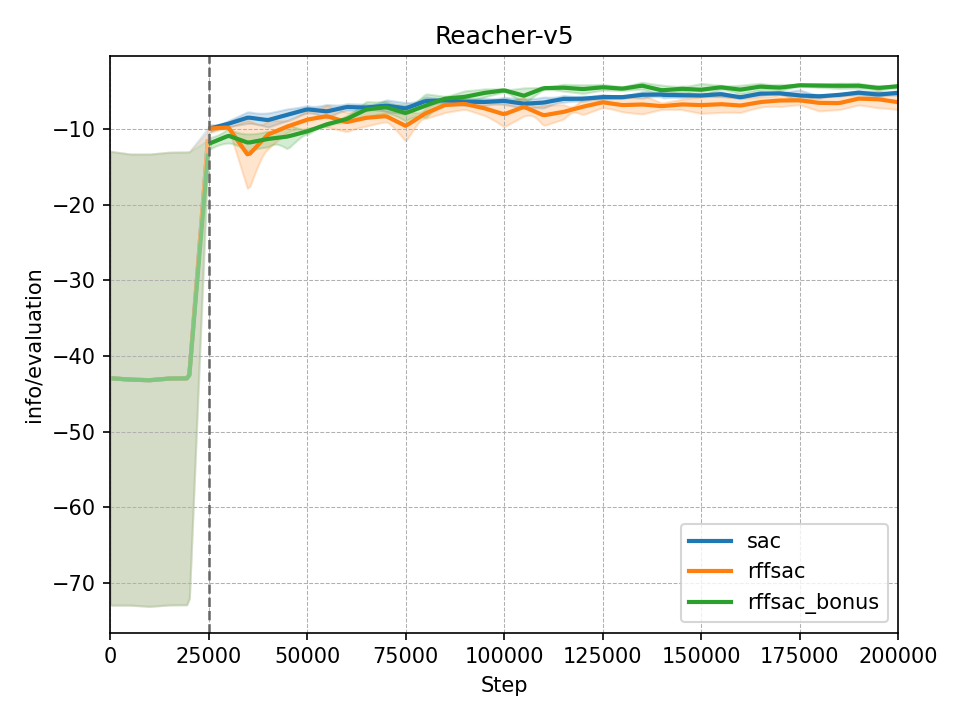}
    \label{fig:ablation_d}
  \end{subfigure}
  \hfill
  \begin{subfigure}[t]{0.32\textwidth}
    \centering
    \includegraphics[width=\linewidth]{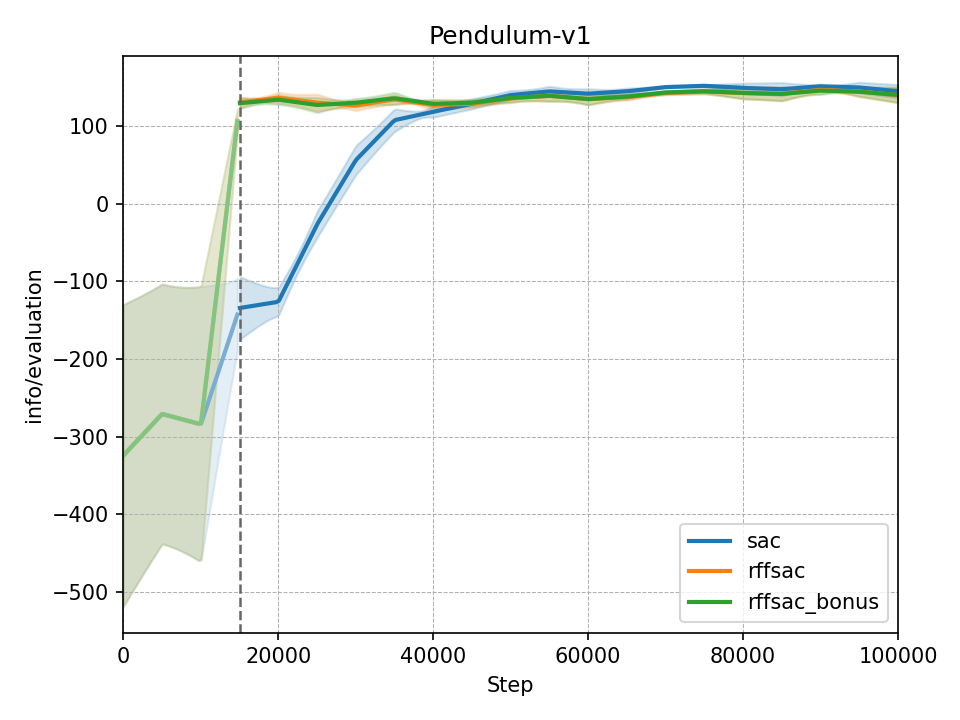}
    \label{fig:ablation_e}
  \end{subfigure}
  \begin{subfigure}[t]{0.32\textwidth}
    \centering
    \includegraphics[width=\linewidth]{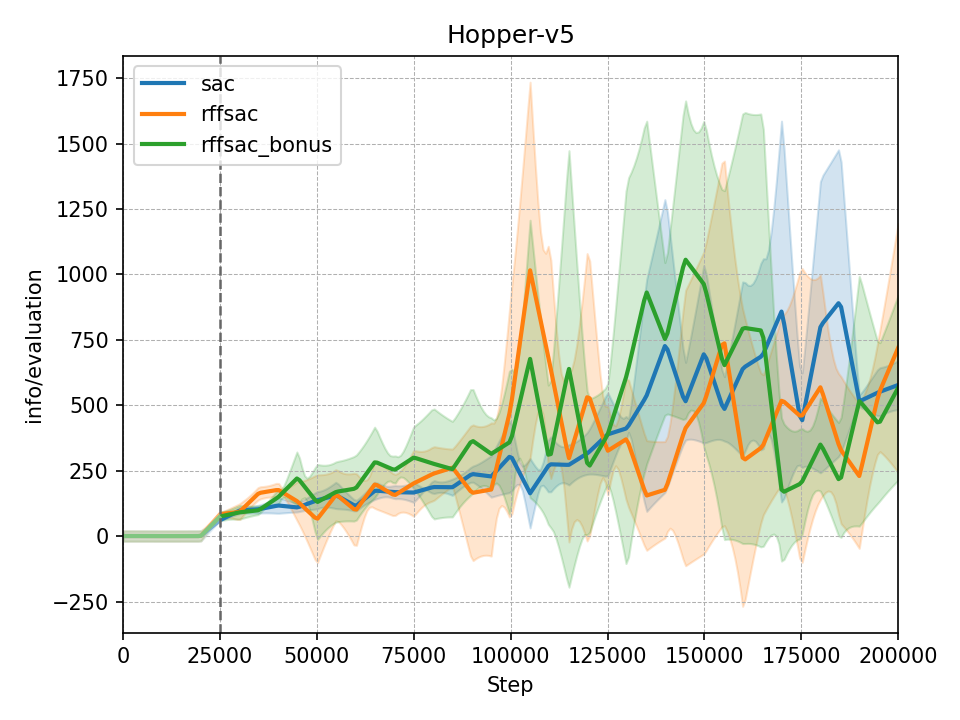}
    \label{fig:ablation_f}
  \end{subfigure}

  \caption{Evaluations over training. Steps before the dash line is uniform random exploration. y-axis is the value of one episode. Results are averaged across 4 different random seeds.}
  \label{fig:five_panel}
\end{figure*}

\end{document}